\documentclass[twoside,11pt]{article}

\usepackage{jmlr2e}

\usepackage{paralist, amsmath, amssymb, color, graphicx, hyperref, algorithm, algpseudocode, float, floatflt, cite, mathtools, bm}
\usepackage{tikz, pgfplots}

\usepackage{enumitem, bigstrut}
\usepackage{booktabs, multirow, threeparttable}
\usepackage{setspace}
\usepackage{mathabx}
\usepackage{thmtools, thm-restate, cleveref}
\usepackage{wrapfig}

\pgfplotsset{compat=1.18}

\newcommand{\ex}{\mathbb{E}}

\newcommand{\ep}{\epsilon}

\usepackage{lastpage}

\ShortHeadings{Near-Optimal Private Tests for Simple and MLR Hypotheses}{Chen, Pasupathy and Awan}
\firstpageno{1}
\begin{document}

\title{Near-Optimal Private Tests for Simple and MLR Hypotheses}

\author{\name Yu-Wei Chen \email chen4357@purdue.edu \\
       \addr Department of Statistics\\
       Purdue University \\
       West Lafayette, IN 47907, USA
        \AND
       \name Raghu Pasupathy \email pasupath@purdue.edu  \\
       \addr Department of Statistics\\
       Purdue University \\
       West Lafayette, IN 47907, USA
       \AND
       \name Jordan Awan \email jaa557@pitt.edu  \\
       \addr Department of Statistics\\
       University of Pittsburgh \\
       Pittsburgh, PA 15260, USA}

\editor{}

\maketitle

\begin{abstract}
    We develop a near-optimal testing procedure under the framework of Gaussian differential privacy for simple,  as well as one- and two-sided tests under monotone likelihood ratio conditions. Our mechanism is based on a private mean estimator with data-driven clamping bounds, whose leading term of the population risk matches that of the non-private sample mean (including the constant) and second term matches the private minimax rate up to logarithmic factors. Using this estimator, we construct private test statistics that achieve the same asymptotic relative efficiency as the non-private, most powerful tests while maintaining conservative type I error control. In addition to our theoretical results, our numerical experiments show that our private tests outperform competing DP methods and offer comparable power to the non-private most powerful tests, even at moderately small sample sizes and privacy loss budgets.
\end{abstract}


\section{Introduction}

Differential Privacy (DP), introduced by \citet{dwork2006calibrating}, protects individual privacy by injecting calibrated randomness into data processing; this noise complicates statistical analysis and can lead to biased or invalid conclusions if ignored \citep{Alexis2020, Christopher2021}. As a result, classical inferential tools have been extended to the DP setting, including hypothesis testing \citep{gaboardi2016differentially, awan2018binomial}, confidence regions \citep{wang2019differentially, covington2021unbiased,wang2025optimal}, and finite-sample inference \citep{awan2025simulation}.

Of all of the statistical tasks, hypothesis testing is one of the most fundamental, being a primary tool in scientific research, as well as having connections to confidence intervals and minimax lower bounds \citep{ibragimovHasminskii1981}.
%
%
Given the crucial role of hypothesis testing in statistics, developing its private counterparts is important for the adoption of DP methods. 

A critical issue in private testing is small-sample performance, often overlooked in the computer science literature but crucial in clinical trials and experimental designs \citep{sakpal2010sample, ledolter2020focus}, where data collection is costly and time-consuming.  More broadly, this concern reflects a central question in DP: how large must a sample be before privacy-induced noise becomes negligible in practice? While \citet{canonne2019structure} established the optimal sample complexity for simple hypothesis testing using a clamped and noised per-sample log-likelihood statistic, their use of a fixed clamping range fails to attain optimal constants, resulting in a loss of effective sample size.

This difficulty points to a deeper, unresolved challenge in private testing: the absence of an exact analogue of the \emph{Neyman–Pearson Lemma} under differential privacy.  While \citet{awan2023canonical} proved that there exists a most powerful DP test for simple hypotheses, there is no known closed form for this test, beyond simple settings such as Bernoulli data \citep{awan2018binomial}. 

In this paper, we address both challenges---small-sample inefficiency and the absence of a Neyman–Pearson–type testing principle under DP---in a one-dimensional setting, which underlies higher-dimensional problems \citep{casella2002statistical, luenberger2008linear}. Our method is built around a private mean estimator that is minimax-rate optimal and preserves non-private asymptotic relative efficiency. This estimator is obtained via a data-dependent clamping rule that is constructed using private quantiles and is inspired by \citet{canonne2019structure} and \citet{huang2021instance}. Leveraging the private mean estimator, our method yields a near-optimal test for simple and monotone likelihood ratio (MLR) hypotheses. Our results are developed in the Gaussian DP framework \citep{dong2022gaussian}, which is increasingly adopted as the state of the art \citep{gomez2025gdpReport} and also implies zero-concentrated DP \citep{bun2016concentrated}, which was used in the 2020 U.S. Decennial Census products \citep{abowd20222020}.

\paragraph{Contributions:}
\begin{itemize}
\item \textbf{Rank-error calibration in DP quantile estimation.}
Our analysis corrects technical errors in \citet{huang2021instance} by identifying the crucial role that bin width serves. In particular, we refine the configuration of the noisy binary search used for DP quantile estimation, obtain a corrected rank error bound, and derive a tighter bound for a fixed failure probability.

\item \textbf{Near-optimal private mean estimation.}
We develop a private mean estimator whose population risk, accounting for both the sampling variability and privacy-induced randomness, matches the DP minimax rates up to log-factors. This estimator thus has the same asymptotic distribution as the sample mean, a key ingredient for our private testing procedure. This mechanism is also of independent interest, outperforming competing methods in our simulation studies.

\item \textbf{Near-optimal private tests for simple and MLR hypotheses.}
Building on these components, we introduce a unified framework for private hypothesis testing under simple, one-sided MLR, and two-sided MLR hypotheses. The resulting tests achieve near-optimal power, attaining the same asymptotic efficiency as their non-private counterparts, and are supported by rigorous theoretical guarantees and practical guidance for implementation, including conservative type I error guarantees.
\end{itemize}

\textbf{Organization:}
Section~2 reviews relevant background on hypothesis testing and Gaussian differential privacy. Section~3 introduces the two core algorithms, \texttt{GDP-Quant} and \texttt{GDP-MeanEst}, and establishes their theoretical properties. Section~4 presents our private testing procedures and establishes their asymptotic properties. Section~5 contains numerical experiments to evaluate the performance of our methods against competing approaches, first comparing 
\texttt{GDP-MeanEst} against other DP mean estimation algorithms and then comparing our DP testing procedure against alternatives. 
Section~6 discusses implications and possible directions for future work.

\textbf{Related Work:}
This work most closely builds upon the contributions of \citet{canonne2019structure} and \citet{huang2021instance}. \citet{canonne2019structure} analyze a randomized, clamped variant of the log-likelihood ratio test, which they showed has optimal sample complexity for simple hypotheses. \citet{huang2021instance} developed an instance-optimal private mean estimation mechanism which first estimates the location of the data using a DP quantile mechanism, clamps the data to this estimated range, and then adds Gaussian noise to the clamped mean. While our data-dependent clamping procedure is most inspired by \citet{huang2021instance}, data-dependent clamping has appeared in several other works \citep{smith2011privacy-preserving,biswas2020coinpress,covington2021unbiased}, with varying details. While not explicitly designed for simple and MLR testing, goodness-of-fit tests such as the Kolmogorov-Smirnov and Cramer-von-Mises tests can be applied in these settings; \citet{awan2025differentially} developed DP analogues of these tests, which we compare against.  

The statistical literature on differentially private hypothesis testing traces back to work on private chi-squared test statistics for genomic and clinical trial data \citep{Duy2009dpClinical, Uhler2013privacyPreserving}. This line of research was later extended to differentially private versions of classical finite-sample tests for categorical data, including tests of independence, goodness-of-fit, and distributional closeness \citep{wang2015revisiting, gaboardi2016differentially, kakizaki2017dpChi, ryan2017aNewClass}. \citet{karwa_finite_2017} derive the first finite-sample differentially private confidence interval for Gaussian data, which can also be used for hypothesis testing.

There are also several DP tests based on robust or distribution-free statistics, which yield powerful results in small samples \citep{couch2019differentially,awan2025differentially}. Furthermore, following the universally optimal binomial test of \citep{awan2018binomial}, a series of works has investigated optimal private tests under various settings \citep{awan2020differentially,awan2023canonical}. However, universally optimal pure-DP tests are known not to exist in larger domains or multi-dimensional settings \citep{brenner2014impossibility, awan2023canonical,awan2022log}. While this paper focuses on the ``central model'' of differential privacy, where the data is held by a trusted curator, there is also a substantial line of research on hypothesis testing under local differential privacy as well \citep{or2018local, gaboardi2018local, jayadev2019testWithout}.

For hypothesis testing with generic distributions, several perspectives—including approximation, clamping, and empirical distributions—have been explored. \citet{wang2018statistical} construct statistical approximating distributions for DP statistics and apply them to testing problems. \citet{awan2025simulation} and \citet{wang2025optimal} use simulation-based inference strategies to produce hypothesis tests and confidence itnervals from privatized test statistics. \citet{kazan2023test} and \citet{pena2025differentially} developed a DP wrapper to combine multiple hypothesis tests each computed on subsets of the data, which have guarantees on the type I error, but limited power \citep{awan2025simulation}.

\section{Background} \label{sec: background}

We introduce the necessary background for hypothesis testing and Gaussian differential privacy.

\subsection{Hypothesis Testing}

Let $\Theta_0,\Theta_1$ be a partition of the parameter space $\Theta$ and $\underline{x} = (x_1, \ldots , x_n) \in \mathcal{X}^n$ be distributed $x_i\overset{\mathrm{i.i.d.}}{\operatorname*{\sim}}f_\theta$, where $\theta\in\Theta.$  
A (randomized) test of $H_0:\theta\in\Theta_0$ versus $H_1:\theta\in\Theta_1$ is a measurable function $\phi:\mathcal{X}^n\to[0,1]$. A test $\phi$ is said to be level-$\alpha$ if $\sup_{\theta\in\Theta_0}\mathbb{E}_{f_\theta}\phi\leq\alpha.$ The power of $\phi$ at $\theta$ is denoted $\beta_\phi(\theta)=\mathbb{E}_{f_{\theta}}\phi$.   

We interpret $\phi(\underline{x})$ is the probability of rejecting the null hypothesis, given an observation $\underline{x} \in \mathcal{X}^n.$ The outcome of a test is either ``Reject'' or ``Fail to Reject'' with respective probabilities $\phi(\underline{x})$, and $1-\phi(\underline{x})$. The following are two useful concepts in hypothesis testing. Let $\{\phi_n\}$ be a sequence of test functions. If $\ex_{f_{\theta_0}} \phi_n \to \alpha$ with $\theta_0 \in \Theta_0$, then $\phi_n$ is said to have asymptotic level $\alpha$. If $\beta_{\phi_n}(\theta_1) \to 1$ with $\theta_1 \in \Theta_1$, then $\phi_n$ is said to be consistent.

\begin{definition}[Relative efficiency, \citealp{van2000asymptotic}]
Let $\varphi_{\nu,n}$ be a test of $H_0:\theta=\theta_0$ versus $H_1:\theta=\Theta_1$, based on $n$ observations, where $\nu\to\infty$ indexes the asymptotics. For each $\nu$, let $n_\nu$ denote the minimal sample size for which $\varphi_{\nu,n}$ satisfies the requirements as follows. \textbf{(i) Pitman setting.}
Under local alternatives $\theta_\nu=\theta_0+h\nu^{-1/2}\in \Theta_1$, $n_\nu$ is the smallest $n$ for which $\varphi_{\nu,n}$ attains asymptotic level $\alpha$ and some power $\gamma\in(\alpha,1)$. \textbf{(ii) Bahadur setting.} For a fixed alternative $\theta_1\neq\theta_0$, $n_\nu$ is the smallest $n$ such that $P_{\theta_1}(\varphi_{\nu,n}=0)\le a_\nu$ with $a_\nu\downarrow 0$. Then, in either setting, for two sequences of tests $\varphi^{(1)}_{\nu,n}$ and $\varphi^{(2)}_{\nu,n}$, the (asymptotic) relative efficiency is $\mathrm{ARE}(\varphi^{(1)},\varphi^{(2)}) = \lim_{\nu\to\infty}\frac{n^{(2)}_\nu}{n^{(1)}_\nu}$.
\end{definition}

\begin{proposition}[Bahadur slope, \citealp{bahadur1967rates}]
If the type II error converges: 
\[-\frac{1}{n}\log P_{\theta_1}\left( \varphi_{\nu,n}^{(i)}=0 \right)\to c^{(i)}(\theta_1),\] then $c^{(i)}(\theta_1)$ is the Bahadur slope, and $\mathrm{ARE} = \frac{c^{(1)}(\theta_1)}{c^{(2)}(\theta_1)}$.
\end{proposition}

While consistency is a minimal requirement, ensuring that a test can distinguish the null from the alternative, relative efficiency more directly captures asymptotic power. Given a most powerful test, another test is asymptotically equivalent if its ARE equals one, whereas lesser values indicate a lower effective sample size compared to the optimal test.

\begin{definition}[Monotone Likelihood Ratio]
A family of densities $\{f(x;\theta): \theta \in \Theta\}$ has a \emph{monotone likelihood ratio} (MLR) in a statistic $t(x)$ if, for any $\theta_2 > \theta_1$, $\frac{f(x;\theta_2)}{f(x;\theta_1)} = g(t(x))$ for some nondecreasing function $g$ of $t(x)$.
\end{definition}

The MLR property is central in classical hypothesis testing, as it allows for the existence of a uniformly most powerful (UMP) test. 
A classical example is the exponential family, with density $f(x\mid\theta) = h(x)\exp\!\left(\omega(\theta)^\top t(x) - A(\theta)\right)$, where $t(x)$ is a sufficient statistic, $\omega(\theta)$ the natural parameter, and $A(\theta)$ the log-partition function. In a one-parameter exponential family with monotone $\omega(\theta)$, the following is increasing in $t(x)$ and thus satisfies the MLR property:
$$
\frac{f(x;\theta_2)}{f(x;\theta_1)} = \exp\!\big((\omega(\theta_2)-\omega(\theta_1)) t(x) - (A(\theta_2)-A(\theta_1))\big).
$$

\subsection{Differential Privacy}

Differential Privacy (DP) \citep{dwork2006calibrating} provides a probabilistic framework for quantifying privacy risks in data analysis. It ensures that the output distribution
of a privacy mechanism changes only marginally when the data of an individual is altered, thereby limiting the information an adversary can infer about that individual. 

    Given a set $\mathcal{X}^n$, which is the collection of all possible databases, a \emph{privacy mechanism} $\mathcal{M}$ is defined as a set of probability measures $\{ \mathcal{M}(\underline{x}) \mid \underline{x} \in \mathcal{X}^n \}$ which take values on a common 
    measurable space $\mathcal{Y}$.


DP is formalized in terms of the similarity between two distributions from the privacy mechanism, when one data point is changed. Different DP frameworks use alternative ways to quantify this ``similarity.'' \citet{dong2022gaussian} uses a hypothesis testing formulation to define Gaussian differential privacy (GDP), which is becoming accepted as the state-of-the-art DP framework \citep{gomez2025gdpReport}:

\begin{definition}[Gaussian Differential Privacy, \citealp{dong2022gaussian}]
A mechanism $\mathcal{M}$\\ satisfies $\ep$-GDP, if for any two neighboring datasets $\underline{x}, \underline{x}' \in \mathcal{X}^n$ differing in one record, any hypothesis test that tries to distinguish whether $\underline{x}$ or $\underline{x}'$ was used by $\mathcal{M}$ has type I error and power $(\alpha, \beta)$ satisfying
$$
\Phi\bigl(\Phi^{-1}(1 - \alpha) - \ep\bigr) \le 1-\beta,
$$
where $\Phi$ is the cumulative distribution function of the standard normal distribution.
\end{definition}

The parameter $\ep$ quantifies that discerning the values for one individual is at least as hard as testing $H_0: N(0,1)$ versus $H_1: N(\ep,1)$. Thus, smaller $\ep$ means outputs are harder to tell apart, thereby yielding stronger protection but typically at the cost of lower utility. 

A common technique to satisfy differential privacy is by an additive noise mechanism. To satisfy GDP, Gaussian noise can be added with scale proportional to the sensitivity of the statistic and inversely proportional to $\ep$:

\begin{proposition}[Gaussian Mechanism, \citealp{dong2022gaussian}]\label{prop:gdp_mechanism}
Let $f: \mathcal{X}^n \to \mathbb{R}^d$. The Gaussian mechanism $\mathcal{M}(D) = f(D) + \mathcal{N}\!\left(0, \tfrac{(\Delta_2 f)^2}{\ep^2} \mathbf{I}_d\right)$ satisfies $\ep$-GDP, where $\Delta_2 f = \max_{D, D'} \| f(D) - f(D') \|_2,$ is the $\ell_2$-sensitivity of $f$.
\end{proposition}

Differential privacy has a few key properties, including composition and invariance to post-processing. Composition concerns itself with the cumulative privacy cost after multiple sequential releases: If $\mathcal{M}_1$ and $\mathcal{M}_2$ satisfy $\ep_1$-GDP and $\ep_2$-GDP respectively, then the sequential release $(\mathcal{M}_1(\underline{x}), \mathcal{M}_2(\underline{x},\mathcal{M}_1(\underline{x})))$ satisfies $\sqrt{\ep_1^2 + \ep_2^2}$-GDP. Invariance to post-processing says that the output of a DP mechanism cannot be made less private by any data-independent procedure:  Let $\mathcal{M}: \mathcal{X}^n \to \mathcal{R}$ be a randomized mechanism that satisfies $\ep$-GDP. Then for any (possibly randomized) measurable function $g: \mathcal{R} \to \mathcal{R}'$, the post-processed mechanism $g \circ \mathcal{M}: \mathcal{X}^n \to \mathcal{R}'$ also satisfies $\ep$-GDP.

\algrenewcommand\algorithmicrequire{\textbf{Input:}}
\algrenewcommand\algorithmicensure{\textbf{Output:}}

\section{Problem Formulation and Basic Algorithms} \label{sec: problem formulation and basic algorithms}

In this section, we first formulate our problem and set up necessary notation. Then, we refine the quantile estimation algorithm of \citet{huang2021instance} and propose our private mean estimation method, which will be applied to obtain near-optimal DP tests in Section~\ref{sec: near-optimal private tests}. 

\subsection{Problem Formulation} \label{sec: problem formulation}

Consider a sensitive i.i.d. dataset $\underline{X} =(X_1,\ldots,X_n)$ with $\ex X_i = \mu$ as our parameter of interest. {While this section addresses general quantile and mean estimation, in our testing setting---particularly in Section~\ref{sec: near-optimal private tests}---$X_i$ refers to a test statistic computed from the $i$-th individual's data, rather than the raw data itself. 

We impose two mild assumptions on the sample distribution: for $i \in [n]$,
\begin{itemize}[leftmargin=*]
    \item[] (A.1) $X_i$ has a continuous distribution with a bounded density, $f$. That is, $f(x) \leq M < \infty$ for all $x \in \mathbb{R}$.
    
    \item[] (A.2) $X_i - \mu$ are subexponential.  
    Namely, there exists scale parameter $s>0$ such that $P(|X_i -\mu|\geq c) \leq 2 \exp(-c/s)$ for all $c>0$.
\end{itemize}


Our intermediate goal is to develop a private mean estimator $\mu_{DP}$, satisfying $\epsilon$-GDP, such that
\begin{align} \label{eq: goal}
    |\mu_{DP}(\underline{X};\ep) - \Bar{X}| = \widetilde{O}_p \left( \frac{s}{\ep n} \right),
\end{align}
which matches the DP term in the minimax lower bound that we derive for mean estimation under subexponential-tailed distributions in Proposition~\ref{prop: GDP minimax mean estimation lower bound}, adapting the proof techniques of \citet{barber2014privacy}. We then apply the private estimator to the data, where the $X_i$'s serve as the log-likelihood ratio and MLR statistic in Section~\ref{sec: near-optimal private tests}.

\subsection{Private Quantile Estimation} \label{sec: private quantile estimation}

Our motivation for applying a private quantile estimation step prior to private mean estimation stems from the relatively low minimax cost of differentially private selection. Our algorithm builds on that of \citet{huang2021instance} with several refinements, and note that a similar approach was independently proposed by \citet{drechsler2022nonparametric}. It works by a series of binary DP tests \citep{blum2013learning}, which iteratively narrow down a region likely to contain the target sample quantile.  

\renewcommand\algorithmicrequire{\textbf{Input:}}
\renewcommand\algorithmicensure{\textbf{Output:}}

\begin{algorithm}[htbp]
\caption{Differentially Private Quantile Selection; \texttt{GDP-Quant}($\mathcal{D},a,b,T,q,\epsilon$)} \label{alg: DP_quantile_selection}
\begin{algorithmic}[1]
\Require data set $\mathcal{D}: x_{(1)} \leq \cdots \leq x_{(n)}$, search range $[a,b]$, number of steps $T$, targeted quantile level $q \in (0,1)$ and privacy budget $\ep>0$
    \State $x_i \leftarrow \max\left\{\min \{x_i, b\}, a \right\}$
    \State left $\leftarrow a$, right $\leftarrow b$
    \State $Z_1,\ldots,Z_T \overset{iid}{\sim}  N\left(0, \frac{T}{\ep^2}\right)$
    \State mid $\leftarrow$ (left + right)/2
    \For{$t=1,\ldots,T$}
        \State noisyCount $\leftarrow \#\{j: a \leq x_j \leq \text{mid}, \forall x_j \in \mathcal{D} \} + Z_t$
        \If{noisyCount $< nq$} 
            \State left $\leftarrow$ mid 
        \Else 
            \State right $\leftarrow$ mid
        \EndIf
        \State mid $\leftarrow$ (left + right)/2
    \EndFor 
\Ensure mid
\end{algorithmic}
\end{algorithm}

To achieve the intermediate goal in \eqref{eq: goal}, we first clarify the rationale for performing private selection prior to private mean estimation}, and identify a key aspect overlooked by \citet{huang2021instance}—the role of bin width. Private selection mechanisms are designed to identify relevant statistics through noisy queries while incurring only a logarithmic cost in the sample size, thereby effectively reducing the sensitivity scale of the data range for subsequent private mean estimation. Prominent examples include \emph{Report Noisy Max}, the \emph{Exponential Mechanism}, and \emph{Above Threshold} \citep{dwork2009complexity, dwork2014algorithmic}.

Regarding the bin width, their search is conducted over integers via the floor function, which fixes the bin width at $1$. When multiple points fall within the same bin, their algorithm can no longer distinguish among them, leading to increased rank errors. This loss of resolution is the main reason the rank-error guarantee claimed in \citet{huang2021instance} is incorrect.

Algorithm~\ref{alg: DP_quantile_selection} serves as a key building block for our subsequent approach. Compared to \citet{huang2021instance}, it allows for non-integer search points, with more significant differences lying in our analysis and parameter selection. For instance, the number of iterations $T$ is treated as a free parameter. This flexibility in the choice of $T$ directly determines the discretization induced by our algorithm. After $T$ iterations, the search range $[a,b]$ is partitioned into subintervals $B_k = [a + (k-1)w,; a + kw]$ for $k \in [2^T]$, where $w = (b-a)/2^T$ is the bin width (discretization error). This makes the choice of $T$ critical to the error trade-off: increasing $T$ reduces discretization error but increases rank error due to smaller per-iteration privacy budgets.

Now, we analyze the accuracy of Algorithm~\ref{alg: DP_quantile_selection} over the DP randomness, holding the data fixed:


\begin{restatable}[GDP-Quant]{lem}{dpQuantile} \label{lem: GDP-Quant}
 Let $N_k = \sum_{i=1}^n \mathbf{1}\{ x_i \in B_k \}$. Assume $N_k\leq1$ for all $k=[2^T]$, then Algorithm \texttt{GDP-Quant} satisfies $\epsilon$-GDP and returns a private quantile with rank error less than $\tau + 1$ with probability at least $1-\beta$, if $\tau = \frac1\ep \sqrt{2T \log \frac{T}{\beta}}$.
\end{restatable}

The statement of Lemma~\ref{lem: GDP-Quant} raises the central role played by the bin width. Although implicit, the assumption on $N_k$—that each bin $B_k$ contains at most one data point—effectively requires the number of steps $T$ to be sufficiently large. Increasing $T$ refines the binning but also enlarges the rank-error bound $\tau$, thereby revealing a trade-off between discretization error and rank error.

\begin{remark}
While \citet{huang2021instance} claimed that their DP quantile  algorithm (similar to Algorithm \ref{alg: DP_quantile_selection}) attains rank error~$\tau$ with probability $\sqrt{\log(b-a)\,\log(\log(b-a)/\beta)/(2\rho)}$, where~$\rho$ is the privacy parameter under concentrated differential privacy, we show in Example~\ref{eg: Huang's issues} in Section \ref{appendix: proof} that this is incorrect. Intuitively, the issue arises because their algorithm can be ``off by one bin'' and has no control over the number of data points falling within each bin. Lemma~\ref{lem: GDP-Quant} provides a corrected analysis. 
\end{remark}

\subsection{Private Mean Estimation} \label{sec: private mean estimation}

This section moves from high-level intuition to our main theoretical result. We explain, at a conceptual level, why our private mean estimation approach can attain minimax-rate optimality up to log factors and perform well with relatively small sample sizes (on the scale of hundreds to thousands). We then contrast our design principles with those of \citet{huang2021instance}, which is based on a Shifted-Clipped-Mean estimator (denoted as \texttt{Shifted-CM}), and \citet{biswas2020coinpress} (\texttt{Coinpress}), and show how these differences lead to suboptimality in their methods.

A heursitc approach to private mean estimation is to first identify a data “center” and then clamp observations to a symmetric ball. Both competing methods follow this strategy: \texttt{Shifted-CM} privately estimates the median, and then estimates a tail quantile of the radius of the centered data, while \texttt{Coinpress} iteratively constructs a confidence ball based on a privatized mean and variance. Although intuitive, symmetric projection can be suboptimal, particularly under skewed or heavy-tailed distributions.

In contrast, \texttt{GDP-MeanEst} estimates a clamping range by applying \texttt{GDP-Quant} to each tail of the data, using only a small fraction of the total privacy budget.  
This allows our clamping region to adapt to the data distribution, while 
minimizing the additive noise needed to protect privacy. 
Our empirical results in Section~\ref{sim: private mean estimation} show that \texttt{GDP-MeanEst} significantly outperforms competing methods.

\begin{algorithm}[htbp] 
\caption{Differentially Private Instance Mean Estimation; \texttt{GDP-MeanEst}$(\mathcal{D},a,b,\ep_q,\ep_m, \eta)$} \label{alg: DP_instance_mean_estimation}
\begin{algorithmic}[1]
\Require data set $\mathcal{D}: x_{(1)} \leq \cdots \leq x_{(n)}$, search range $[a,b]$, privacy budgets $\ep_q, \ep_m>0$ and $\eta>2$
    \State $x_i \leftarrow \max\left\{\min \{x_i, b\}, a \right\}$
    \State $T \leftarrow \lceil \log_2 [(b-a) n^\eta] \rceil$  
    \State $q_l = (\tau + 2)/n$ and $q_u = 1 - (\tau + 1)/n$, where $\tau = \sqrt{2T\log\frac{T}{n^{2-\eta}}} / \epsilon_q$
    \State {$x_{(nq_l)}^{DP} \leftarrow \texttt{GDP-Quant}(\mathcal{D}, a, b, T, q_l, \ep_q)$}
    \State {$ x_{(nq_u)}^{DP} \leftarrow \max\{\texttt{GDP-Quant}(\mathcal{D},a,b,T,q_u,\ep_q), x_{(nq_l)}^{DP}\} $}
    \State $\widetilde{x}_i = \max\left\{ \min\{x_i, x_{(nq_l)}^{DP}\}, x_{(nq_u)}^{DP}\right\}$
    \State $z_m \sim N \left(0, (x_{(nq_u)}^{DP} - x_{(nq_l)}^{DP} )^2/(n^2 \ep_m^2) \right)$
\Ensure $\frac{1}{n} \sum_{i=1}^n \widetilde{x}_i + z_m$
\end{algorithmic}
\end{algorithm}

\begin{proposition}[GDP Guarantee]\label{prop: both algorithms are GDP}
Algorithm \texttt{GDP-MeanEst}$(\mathcal{D},a,b,\ep_q,\ep_m, \eta)$ satisfies $\ep$-GDP, where $\ep^2=2\ep_q^2+\ep_m^2$.
\end{proposition}

More precisely, \texttt{GDP-MeanEst}, described in
Algorithm~\ref{alg: DP_instance_mean_estimation} applies 
\texttt{GDP-Quant} as follows: It specifies the number of search steps as $T = \left\lceil \log_2 \bigl[(b-a) n^\eta\bigr] \right\rceil$ and targets the quantile levels $q_l = (\tau + 2)/n$ and $q_u = 1 - (\tau + 1)/n$. The key intuition of the choices is that we intentionally target quantiles slightly inward from the extremes in a principled way, leveraging the subexponential behavior of the data (even without prior knowledge of the dataset). This ensures, with high probability, that the search remains within the true data range, which in turn keeps the sensitivity well-controlled and allows for accurate estimation.

Algorithm~\ref{alg: DP_instance_mean_estimation} treats the initial search range $[a,b]$ and the privacy budget allocation as free parameters; specific choices are imposed in Definition~\ref{def: setting of GDP-MeanEst} to achieve optimality matching the private minimax lower bound up to logarithmic factors.

\begin{definition} \label{def: setting of GDP-MeanEst}
    Let $\underline{X}$ be a $n$-dimensional data vector drawn from a distribution, and $\ep$ as the GDP parameter. Define
    \begin{align*}
        \mu_{DP} (\underline{X};\ep) = \texttt{GDP-MeanEst}(\underline{X},a,b,\ep_q,\ep_m,\eta)
    \end{align*}    
    as our private mean estimator, where $a$, $b$, $\ep_q$, and $\ep_m$ are speficied as follows:  $[a, b] = [-v(\log n)^{p}, b(n) = v(\log n)^p]$ is the initial search range and $T = \lceil \log_2 [(b-a) n^\eta] \rceil$ is the number of iterations, with $p>1$ and $\eta, v>0$, $\ep_q = \ep /(\log n)^k$ and $\ep_m = \ep \sqrt{1-\frac{2}{(\log n)^{2k}}}$ with constant $k \in (0,1]$ such that $\ep^2 = 2\ep_q^2 + \ep_m^2$. 
\end{definition}



Using Lemma~\ref{lem: GDP-Quant}, the next theorem quantifies the gap between our private mean estimator and its non-private counterpart. The bound accounts for both the privacy-induced randomness and the variability due to sampling. 

\begin{restatable}[\texttt{GDP-MeanEst} Utility]{thm}{dpInstMeanEstUtility} \label{thm: GDP-MeanEst Utility}
    Assume $\underline{X}$ satisfies (A.1) and (A.2). Given by Definition~\eqref{def: setting of GDP-MeanEst}, the private mean estimate $\mu_{DP} (\underline{X},\ep)$ has absolute error
    \begin{align*}
    |\mu_{DP}(\underline{X};\ep) - \bar{X}| &= O_p\left(\frac{s(\log n)^{1+k} \sqrt{T\log T}}{\ep n}\right) \\
    &= \widetilde{O}_p \left( \frac{s}{\ep n} \right),
    \end{align*}
    where $T(n,v,p,\eta) = \Theta(\log v + p \log_2 (\log n) + \eta \log_2 n)$. Marginally, $\mu_{DP} (\underline{X})$ and $\bar{X}$ share the same asymptotic distribution under the central limit theorem.
\end{restatable}

\begin{proof}\textbf{Sketch.}
    Lemma~\ref{lem: subexpoenntial assumed} shows that Assumptions (A.1)–(A.2) imply condition (B) with high probability. Combined with Lemma~\ref{lem: GDP-Quant}, this yields a high-probability bound for \texttt{GDP-Quant} under subexponential data. The $O_p$ result then follows by decomposing the error into Gaussian noise and rank error terms as $n$ increases.
\end{proof}

\begin{remark} \label{remark: initial range of mu}
    If an initial range for $\mu$ is available, say $l_\mu \leq \mu \leq u_\mu$, the search range can be set as $[a, b] = [l_\mu-v(\log n)^{p}, b(n) = u_\mu+v(\log n)^p]$. This improves finite-sample performance, while Theorem~\ref{thm: GDP-MeanEst Utility} still holds.
\end{remark}

Theorem~\ref{thm: GDP-MeanEst Utility} is the main technical result of the paper. It characterizes the role of each parameter and matches the GDP minimax lower bound up to logarithmic factors (see Proposition~\ref{prop: GDP minimax mean estimation lower bound}). By contrast, the approach of \citet{huang2021instance} does not attain this rate even in the Gaussian mean estimation setting, as it fixes the bin width at $O(1/\sqrt{n})$. 

It is important to note that the absolute error bound in Theorem~\ref{thm: GDP-MeanEst Utility} is stated at the population level, accounting for both privacy-induced randomness from Algorithm~\ref{alg: DP_instance_mean_estimation} and sampling variability from subexponential data. Under subexponential distributions, the trimmed sensitivity grows only logarithmically with the initial search range, a consequence of our design that allows the range to adapt to the underlying distribution, while clamping only a small fraction of observations (with $v$ a user-chosen constant). Finally, our privacy budget allocation emphasizes accurate mean estimation while ensuring that the private quantile steps contribute only a logarithmic dependence on the sample size.

\section{Near-Optimal Private Tests} \label{sec: near-optimal private tests}

In this section, we present our results on private testing, which apply \texttt{GDP-MeanEst} developed in Section~\ref{sec: problem formulation and basic algorithms}. Our goal is to show that the resulting DP tests achieve the same \emph{asymptotic relative efficiency} as their non-private, most powerful counterparts, and therefore have asymptotically optimal power. In this paper, by near-optimality we mean that our method matches the non-private benchmark at the level of the leading constant, while the additional error due to privacy is controlled up to logarithmic factors relative to the DP minimax rate established earlier in Section~\ref{sec: problem formulation}. This contrasts with existing approaches: the private likelihood-ratio–based tests of \citet{canonne2019structure} do not attain $\mathrm{ARE}=1$ due to fixed clamping, while applying the normal mean estimation method of \citet{huang2021instance} to log-likelihood statistics leads to additional inefficiency arising from discretization.

The log-likelihood ratio statistic is central to hypothesis testing. Let $P$ and $Q$ be distinct probability measures on $(\mathcal{X},\mathcal{A})$, mutually absolutely continuous. Let $x_i$ be i.i.d. samples from a distribution and $\ell(x_i; P,Q) = \log\frac{P(x_i)}{Q(x_i)}$ be the log likelihood ratio. For convenience, we denote $\underline{\ell}(\underline{x};P,Q) = (\ell(x_1; P,Q),\ldots,\ell(x_n; P,Q))$ as the $n$-dimensional log-likelihood ratio vector.

To derive the non-private optimal test for distinguishing distribution $P$ and $Q$, the  {Neyman--Pearson Lemma} considers the \emph{log-likelihood ratio} statistic:

\[\mathrm{LLR} (\underline{x}; P,Q) = \sum_{i=1}^n \ell(x_i; P,Q).\]

\subsection{Simple Hypothesis} \label{sec: gdpSimpleHT}

For simple hypotheses, likelihood-based tests provide exact control of error rates and large-deviation behavior, forming the basis of asymptotic efficiency theory. In this section, we first show that any non-trivial fixed clamping of LLR tests such as \texttt{ncLLR} in \citet{canonne2019structure}, necessarily incurs a strict loss in asymptotic relative efficiency. We then prove that $\mu_{DP}(\underline{\ell};\ep)$ attains the same asymptotic efficiency as the non-private {Neyman–Pearson test}.

\begin{restatable}{prop}{AREofncLLR}\label{prop: AREofncLLR}
For $H_0:P$ versus $H_1:Q$, let $Y=g(X):=[\ell(X)]_a^b$ be the clamped log-likelihood ratio for some $a<0<b$, and let $P_Y$ and $Q_Y$ denote its laws under $P$ and $Q$.  Assume clamping is nontrivial: $P(\ell(X)\notin[a,b])>0$ and $Q(\ell(X)\notin[a,b])>0$, and that $\ell$ is non-constant on each of the sets $\{x:\ell(x)\le a\}$ and $\{x:\ell(x)\ge b\}$. 
Then any test based on $(Y_1,\dots,Y_n)$, with additive noise independent of $n$, has Bahadur slope at most $2\mathrm{KL}(Q_Y\|P_Y)$, where $\mathrm{KL}(Q\|P)=\int \log\!\left(\frac{dQ}{dP}\right)\,dQ$ is the Kullback--Leibler divergence. Moreover,
$$
\mathrm{KL}(Q_Y\|P_Y) < \mathrm{KL}(Q\|P),
$$
so every $Y$-based test has Bahadur ARE $<1$ relative to the classical LLR test.
\end{restatable}

\begin{remark}
The non-constant assumption on the likelihood rules out the extreme cases in which $P$ and $Q$ are proportional on the clamp region and clamping would cause no information loss. 
\end{remark}

Proposition~\ref{prop: AREofncLLR} shows that, while \texttt{ncLLR} achieves optimal sample complexity up to constant factors, it incurs a loss in effective sample size. This highlights an inherent limitation of approaches such as \citet{canonne2019structure}, where fixed clamping is central to the design. In contrast, our approach uses minimal and data-adaptive clamping, which preserves more information from the data.

The {Neyman--Pearson Lemma} states that, for testing $H_0:\theta=\theta_0$ versus $H_1:\theta=\theta_1$, the (possibly randomized) likelihood ratio test $\psi$ calibrated to satisfy $\mathbb{E}_{f_{\theta_0}}\psi=\alpha$ is the most powerful (MP) level-$\alpha$ test. With the utility result in Theorem~\ref{thm: GDP-MeanEst Utility}, we can establish the following theorem that recovers the same $\mathrm{ARE}$ as the non-private optimal test, providing a stronger notion of optimality.

\begin{restatable}[GDP Simple Hypothesis Test]{thm}{gdpSimpleHT} \label{thm: gdpSimpleHT}
For testing $ H_0:\theta = \theta_0$ against $H_1:\theta = \theta_1$ ($\theta_0<\theta_1$), if the data likelihood under $H_0$ and $H_1$ is $f(\underline{x};\theta_0)$ and $f(\underline{x};\theta_1)$ respectively and assume (A.1) and (A.2) for $\ell(x_i; f_{\theta_0}, f_{\theta_1}) = \log \left( f_{\theta_0}(x_i) /f_{\theta_1} (x_i) \right)$, for all $i \in [n]$. Consider $\mu_{DP}(\underline{\ell},\ep) = \texttt{GDP-MeanEst}(\underline{\ell},a,b,\ep_q,\ep_m,\eta)$ as defined in Definition~\ref{def: setting of GDP-MeanEst}, where $\underline \ell$ is the vector of $\ell(x_i;f_{\theta_0},f_{\theta_1})$. T
hen,
$$
\phi(\underline{x}) :=
\begin{cases}
1, &  \mu_{DP}(\underline{\ell},\ep) > k(n,\ep)\\
\gamma, & \mu_{DP}(\underline{\ell},\ep)  = k(n,\ep) \ , \\
0, & \mu_{DP}(\underline{\ell},\ep)  < k(n,\ep)
\end{cases}
$$
is a level-$\alpha$ test that satisfies $\ep$-GDP and has Bahadur $\text{ARE}(\phi, \psi) = 1$ when $k(n,\ep)$ is chosen such that $\ex_{\theta_0} \phi = \alpha$. 
\end{restatable}


Theorem~\ref{thm: gdpSimpleHT} serves as a GDP analogue of the {Neyman--Pearson Lemma}, achieving the same asymptotic efficiency. The key idea behind Theorem~\ref{thm: gdpSimpleHT} can be understood through the lens of asymptotic relative efficiency (ARE) via the Bahadur slope. In the simple hypothesis setting, $\mathrm{ARE}$ captures not only the power of a test at a fixed alternative, but also how well the test maintains power as the alternative approaches the null at a certain rate. A test is considered stronger in this sense if it can still reliably distinguish the hypotheses even when they become increasingly similar.
        
The Bahadur slope formalizes this idea by quantifying the rate at which the test’s error probability decays. Proposition~\ref{prop: AREofncLLR} shows that any fixed clamping strategy degrades this rate (under the stated conditions), meaning the test loses its ability to distinguish hypotheses when they converge too quickly.
        
In contrast, our method avoids this limitation by using minimal and adaptive clamping, which preserves more information from the data. As a result, it achieves $\mathrm{ARE} = 1$, indicating that it retains the same asymptotic efficiency as the nonprivate test in this stronger sense.

\begin{remark} 
    It is important to note that all of the results in Section~4 only require the \emph{test statistics} to have sub-exponential tails, rather than the underlying data distribution. Consequently, even when the data are heavy-tailed (e.g., Cauchy), the framework may still apply provided that the constructed statistics exhibit suitable concentration behavior.
\end{remark}

\subsection{One-sided Hypothesis with MLR} \label{sec: gdpOneSidedHT}

One-sided hypotheses are also of particular interest, as many applications---such as clinical trials and design of experiments---aim to detect a strictly positive or negative effect. Moreover, many common statistical models, including one-parameter exponential families, satisfy MLR, under which the likelihood ratio test yields a UMP test for one-sided hypotheses.

Let $\Phi$ be a set of tests. Then, $\psi\in\Phi$ is the UMP level $\alpha$ test among $\Phi$ for $H_0:\theta\in\Theta_0$ versus $H_1:\theta\in\Theta_1$ if (i) $\sup_{\theta\in\Theta_0} \beta_{\psi}(\theta)\leq\alpha$ and (ii) for any $\phi\in\Phi$ such that $\sup_{\theta\in\Theta_{0}}\beta_{\phi}(\theta)\leq\alpha$ we have ${\beta_{\psi}}(\theta) \geq {\beta_{\phi}}(\theta)$ for all $\theta\in\Theta_1$.

For testing $H_0:\theta\le\theta_0$ against $H_1:\theta>\theta_0$, if $\{f(x;\theta)\}$ has a monotone likelihood ratio in $t(x)$, then the test $\phi$ that rejects for large values of $t$---with a possibly randomized cutoff---and satisfies $\ex_{f_{\theta_0}} \phi = \alpha$ is the UMP level-$\alpha$ test. This is the \emph{Karlin--Rubin Theorem}.  With the utility result in Theorem~\ref{thm: GDP-MeanEst Utility}, we have the following analogous result.

\begin{restatable}[GDP One-Sided Hypothesis Test with MLR]{thm}{gdpOneSidedHT} \label{thm: gdpOneSidedHT}
For testing $H_0:\theta\le\theta_0$ against $H_1:\theta>\theta_0$, if $\{f(x;\theta)\}$ has a monotone likelihood ratio in $t(x)$ and assume (A.1) and (A.2) for $t(x)$. Consider $\mu_{DP}(\underline{t},\ep) = \texttt{GDP-MeanEst}(\underline{t},a,b,\ep_q,\ep_m,\eta)$ as defined in Definition~\ref{def: setting of GDP-MeanEst}, where $\underline{t} =(t(x_1),\ldots,t(x_n))$ is the $n$-dimensional MLR statistic vector. Then,
\begin{align*} 
\phi(\underline{x}) =
\begin{cases}
1, & \mu_{DP}(\underline{t},\ep) > k(n,\ep)\\
r, & \mu_{DP}(\underline{t},\ep) = k(n,\ep)\\
0, & \mu_{DP}(\underline{t},\ep) < k(n,\ep),
\end{cases}    
\end{align*}
is a level-$\alpha$ that satisfies $\ep$-GDP and and has Pitman $\text{ARE}(\phi, \psi) = 1$ when $k(n,\ep)$ is chosen such that $\ex_{\theta_0} \phi = \alpha$.
\end{restatable}


\subsection{Two-sided Hypothesis with Exponential Family} \label{sec: gdpTwoSidedHT}

Although there is no UMP test for a two-sided hypothesis, we can derive optimality under a smaller class of tests; classically, tests are restricted to be unbiased: A test $\phi$ is \emph{unbiased} if $\beta_\phi(\theta) \le \alpha$ under $H_0$ and $\beta_\phi(\theta) \ge \alpha$ under $H_1$. 

For testing $H_0:\theta=\theta_0$ against $H_1:\theta\neq\theta_0$, if $\{f(x;\theta)\}$ is a one-parameter exponential family, then the test $\phi$ that rejects for extreme values of the sufficient statistic $t(x)$---with possibly randomized cutoffs---and satisfies the level-$\alpha$ condition $\ex_{f_{\theta_0}}\phi=\alpha$ together with the unbiased condition $\ex_{f_{\theta_0}}[\phi(x)\, t(x)]=\alpha\,\ex_{f_{\theta_0}}[t(x)]$ is the UMP unbiased level-$\alpha$ test. Combined with the utility result in Theorem~\ref{thm: GDP-MeanEst Utility}, we obtain the following analogous result.

\begin{restatable}[GDP Two-Sided Hypothesis Test with Exponential Family]{thm}{gdpTwoSidedHT} \label{thm: gdpTwoSidedHT}
For testing $H_0:\theta=\theta_0$ against $H_1:\theta\neq\theta_0$, if $\{f(x;\theta)\}$ belongs to a one-parameter exponential family distribution and assume (A.1) and (A.2) for $t(x)$. Consider $\mu_{DP}(\underline{t},\ep) = \texttt{GDP-MeanEst}(\underline{t},a,b,\ep_q,\ep_m,\eta)$ as defined in Definition~\ref{def: setting of GDP-MeanEst}, where $\underline{t} =(t(x_1),\ldots,t(x_n))$ is the $n$-dimensional MLR statistic vector. Then,
\begin{align*} 
\phi(\underline{x}) =
\begin{cases}
1, & \mu_{DP}(\underline{t},\ep) < k_{l}(n,\ep)\\
1& \mu_{DP}(\underline{t},\ep) > k_{u}(n,\ep)\\
r_a, & \mu_{DP}(\underline{t},\ep) = k_{l}(n,\ep)\\
r_b, & \mu_{DP}(\underline{t},\ep) = k_{u}(n,\ep)\\
0, & k_{l}(n,\ep) \le \mu_{DP}(\underline{t},\ep) \le k_{u}(n,\ep)
\end{cases}
\end{align*}
is an asymptotically unbiased level-$\alpha$ test that satisfies $\ep$-GDP and has Pitman $\text{ARE}(\phi, \psi) = 1$ when $k_{l}(n,\ep)$ and $k_{u}(n,\ep)$ are chosen such that 
\[\ex_{\theta_0} [1\{\mu_{DP}(\underline{t},\ep) < k_l(n,\ep)\}] = \ex_{\theta_0} [1\{\mu_{DP}(\underline{t},\ep) > k_u(n,\ep)\}] = \frac{\alpha}{2} + o(1).\]
\end{restatable}



In Theorem~\ref{thm: gdpTwoSidedHT}, the test $\phi$ is asymptotically level-$\alpha$ because of its asymptotic unbiasedness. However, any bias in its local power expansion---equivalently, any deviation from the derivative-zero condition at $\theta_0$---is of smaller order and therefore does not affect the Pitman efficiency. In practice, a conservative type I error guarantee can be achieved by Monte Carlo calibration, detailed in the next section.

\section{Simulation}

In this section, we first demonstrate the strength of \texttt{GDP-MeanEst} against other private mean estimators. Then, we apply \texttt{GDP-MeanEst} to simple and MLR hypothesis testing problems and compares their performance with other private testing methods. Throughout the simulation section, we fix type I error to be $0.05$, privacy budgets  $\ep \in \{0.5, 1, 2\}$ and vary sample sizes. Each configuration is replicated $1{,}000$ times. The source codes are available at \url{https://github.com/y-w-chen/GDP_NearOptimTest}.

For conservative p-value calculation, we adopt the Monte Carlo procedure from \citet{barber2022testing}: if the test statistic such as the log likelihood ratio test is determined such that smaller values are seen as evidence against the null. Let $\mathbb{I}_m := \mathbf{1}\!\left\{ t(\underline{x}^{(m)}) \ge t(\underline{x}) \right\}$. Given $M$ i.i.d. draws from the null distribution, $\underline{x}^{(1)},\ldots,\underline{x}^{(M)}$, a p-value is,  
\begin{align} \label{eq: p-value}
\text{p-val}_t(\underline{x},\underline{x}^{(1)},\ldots,\underline{x}^{(M)}) := \frac1{M+1} \left(1+\sum_{m=1}^M \mathbb{I}_m \right).
\end{align}

For methods not originally designed under Gaussian DP, we replace their noise-adding distributions (Laplace or Tulap) with Gaussian noise calibrated under its composition rules, so that all methods achieve equally tight privacy guarantees. Additional simulation results can be found in Appendix~\ref{appendix: simulation}.

\subsection{Private Mean Estimation} \label{sim: private mean estimation}

In this section, we compare five private mean estimation methods: (i) \texttt{GDP-MeanEst}, (ii) \texttt{CoinPress} \citep{biswas2020coinpress}, (iii) \texttt{Shifted-CM} \citep{huang2021instance}, (iv) \texttt{Karwa-Vadhan} \citep{karwa_finite_2017} and (v) mean estimation with data-dependent clamps that grow with the sample size, \texttt{Naive-DD}. The non-private sample mean is used as a benchmark. 

We consider three data-generating distributions: a Gamma distribution with shape parameter $2$ and rate $0.5$, a Logistic distribution with location $5$ and scale $2$, and a Gaussian distribution with mean $3$ and variance $1$. The sample size ranges from $100$ to $100{,}000$. The Gamma and Logistic distributions are chosen to represent, respectively, skewed and heavy-tailed settings.

\begin{figure}[t]
    \centering
    \includegraphics[width=0.75\linewidth]
    {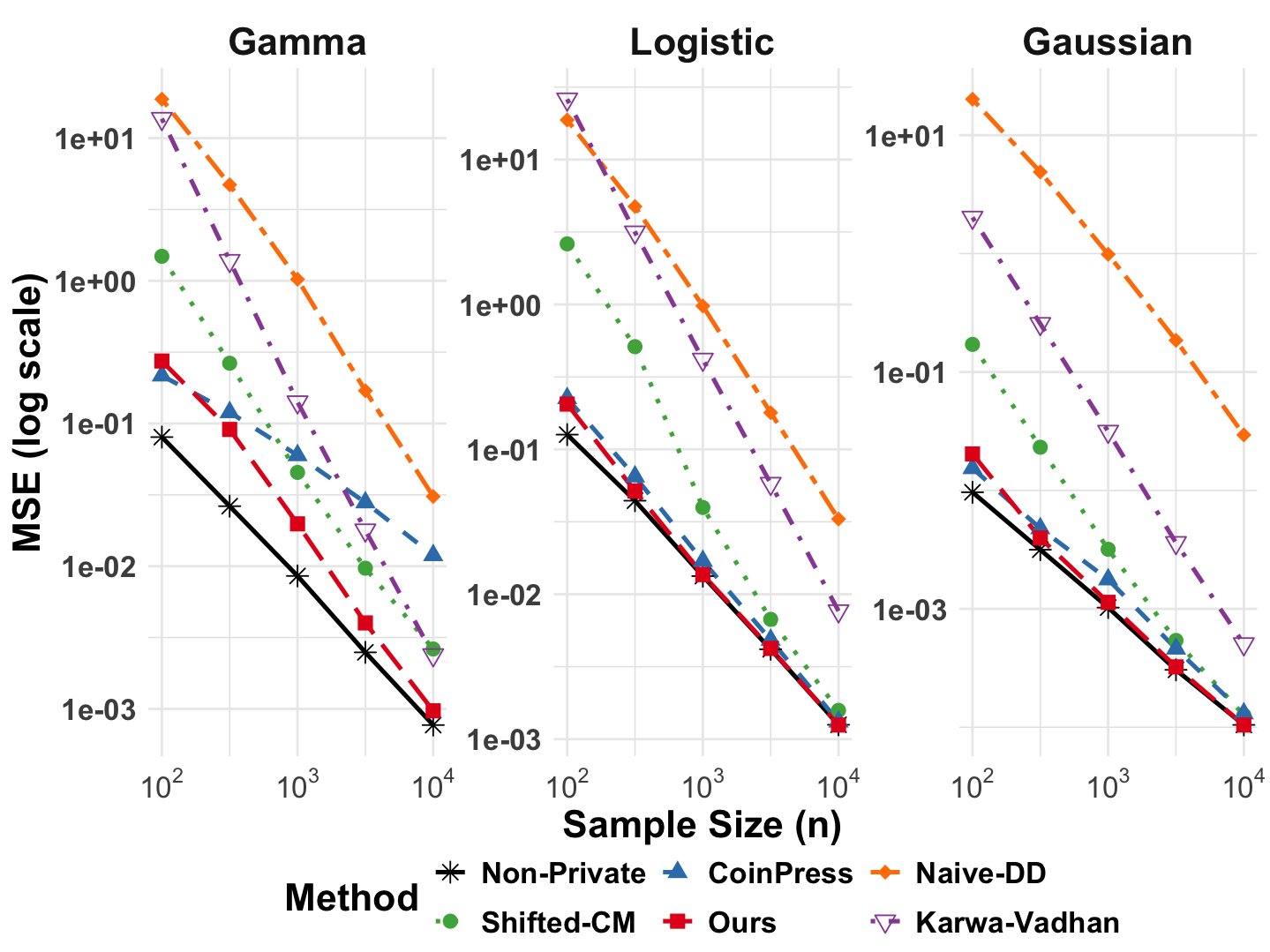}
    \caption{GDP mean estimation comparison ($\epsilon=1$)}
    \label{fig:GDP_MeanEst_Comparison_ep1} 
\end{figure}

Figure~\ref{fig:GDP_MeanEst_Comparison_ep1} shows that our private mean estimator outperforms competing methods across all three data distributions for sample sizes $n \ge 10^{2.5}$. Moreover, our method matches the slope of the non-private benchmark, indicating the same convergence rate. It also achieves consistently smaller error across sample sizes; on the log–log scale, this appears as a downward shift of the curve, reflecting a smaller leading constant. As the sample size increases, the gap between our method and the non-private benchmark narrows down quickly, suggesting that our method approaches the leading constant of the non-private procedure. This behavior is consistent with the $\widetilde{O}_p$ bound established in Theorem~\ref{thm: GDP-MeanEst Utility}.

\subsection{Simple Hypothesis} \label{sim: simple hypothesis}
We study a simple hypothesis testing problem with $H_0:$ $t$ distribution with d.f.\ $=1$ and $H_1:$ equally weighted mixture of two $t$ distributions, one with d.f.\ $=1$ and the other with d.f.\ $=1.1$ and noncentrality parameter $0.1$. We compare five methods: (i) our proposed $\mu_{DP}(\underline{\ell};\ep)$, (ii) \texttt{ncLLR} \citep{canonne2019structure}, (iii) \texttt{K-V} \citep{karwa_finite_2017}, (iv) the private Kolmogorov--Smirnov (KS) test, and (v) the private Cramér--von Mises (CvM) test \citep{awan2025differentially}. Note that, in this and the following simulations, the sample size ranges from $100$ to $3{,}200$, and the $p$-values for all the methods are computed using \eqref{eq: p-value}.

\begin{figure}[t]
    \centering
    \includegraphics[width= 0.75\linewidth]
    {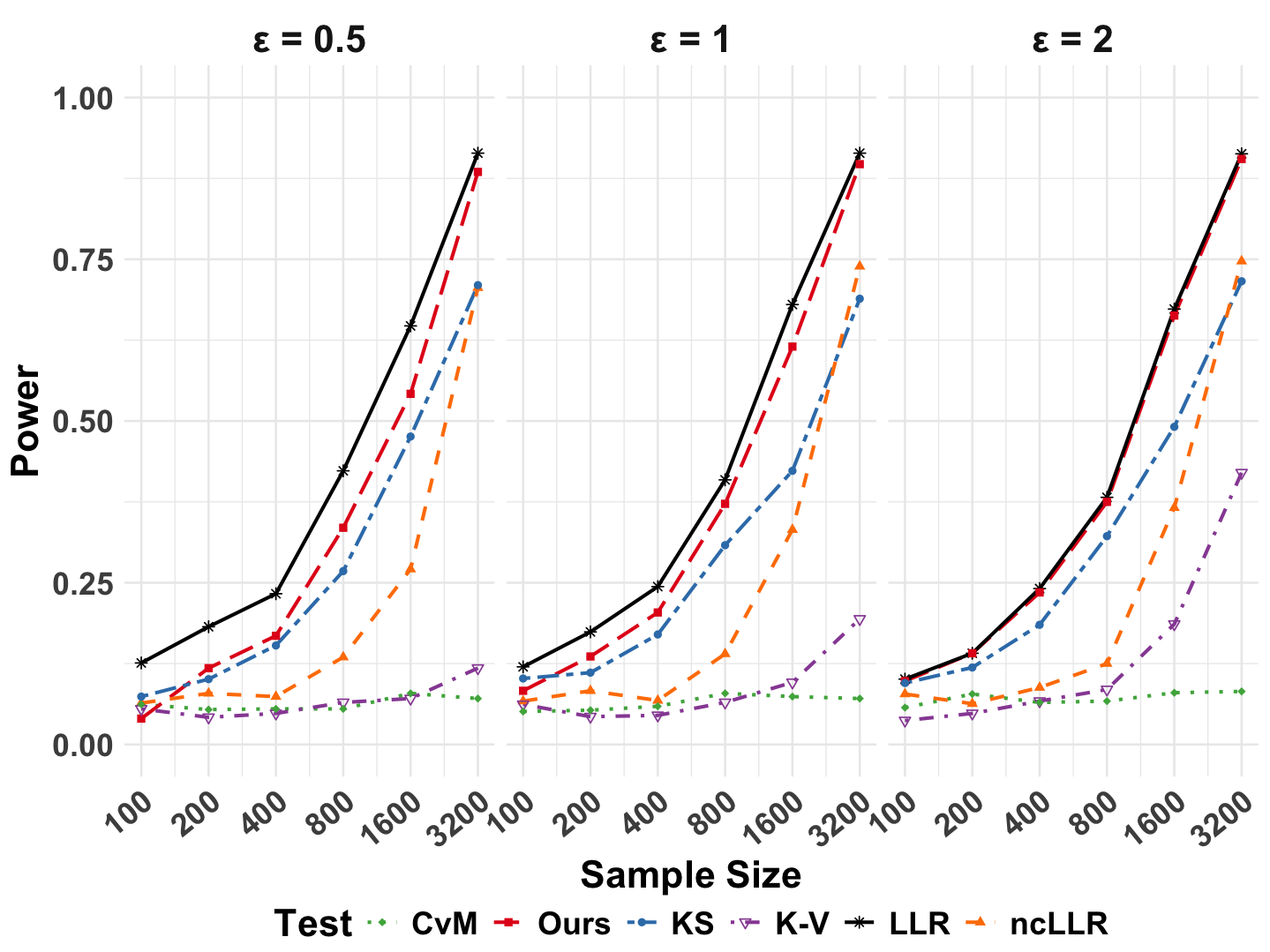}
    \caption{Simple hypothesis under $t$ data}
    \label{fig:SimpleHT_MixtureT}
\end{figure}   

Figure~\ref{fig:SimpleHT_MixtureT} shows that our testing method achieves substantially higher power than the competing methods and approaches the non-private \texttt{LLR} as the sample size increases, indicating that it effectively captures the signal in the mixture $t$ alternatives even under privacy constraints. In contrast, several competing methods exhibit noticeably lower power; in particular, \texttt{CvM} and \texttt{K-V} struggle to distinguish the mixture t alternatives. 

\subsection{One-Sided Hypothesis with MLR} \label{sim: one-sided hypothesis}

Consider a one-sided hypothesis testing problem with $H_0: N(0,1)$ versus $H_1: N(\theta_1,1)$, where $\theta_1>0$. We compare three methods: $\mu_{DP}(\underline{t};\ep)$, the private Kolmogorov--Smirnov (KS) test, and the private Cramér--von Mises (CvM) test. Note that \texttt{ncLLR} is not applicable to one-sided testing, as its test statistic requires specification of both the null and alternative distributions. 

\begin{figure}[t]
    \centering
    \includegraphics[width=0.75\linewidth]
    {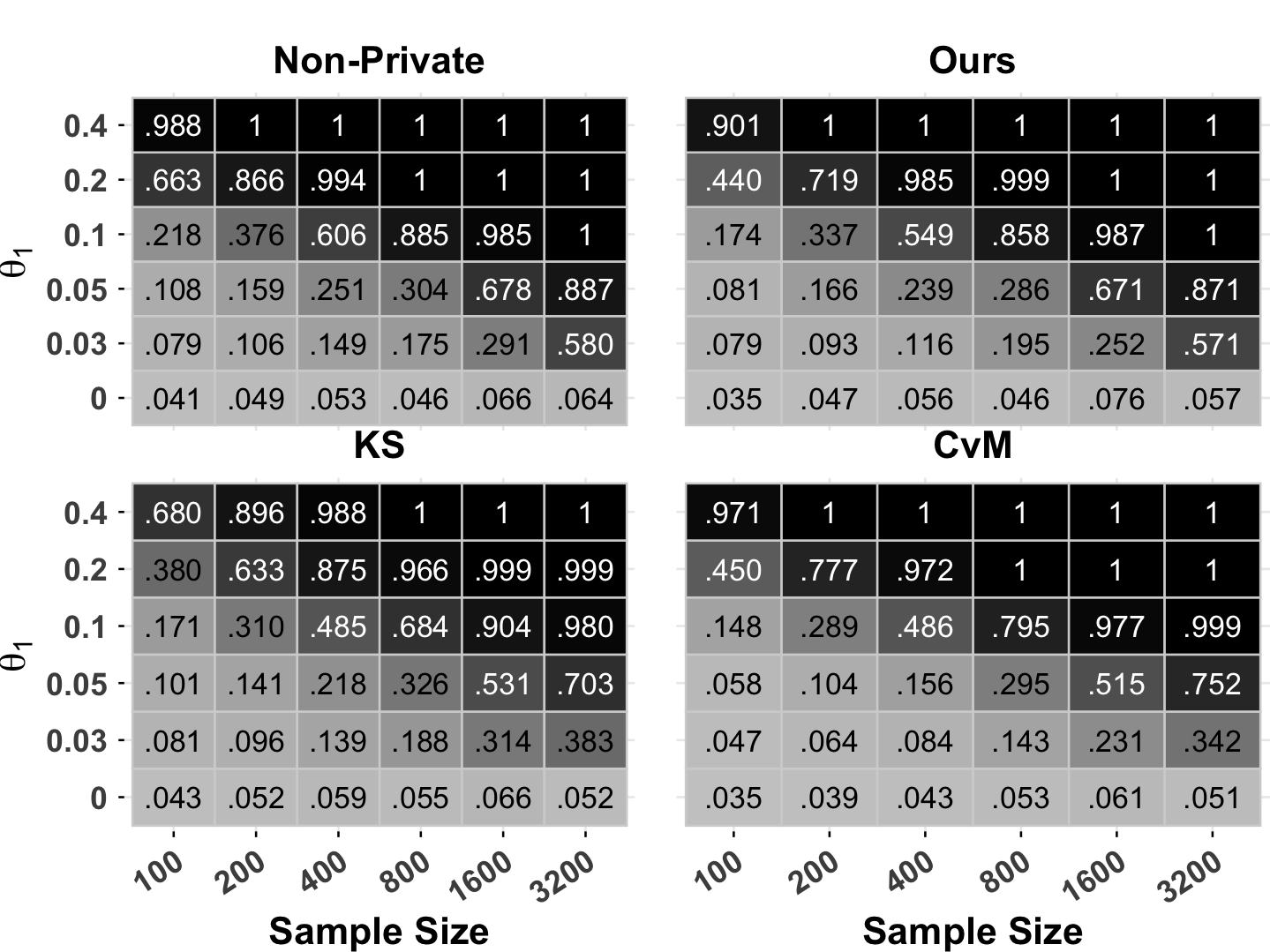}
    \caption{One-sided hypothesis under Gaussian data ($\epsilon=1$)}
    \label{fig:OneSidedHT_Normal_ep1}
\end{figure}

Figure~\ref{fig:OneSidedHT_Normal_ep1} shows that for sample sizes $n \ge 200$, our proposed $\mu_{DP}(\underline{t};\ep)$ attains the highest power among the DP methods. At a sample size of 100, $\mu_{DP}(\underline{t};\ep)$ is comparable to \texttt{CvM}, while \texttt{KS} achieves slightly higher power. As the sample size increases, the power of our method steadily improves and closely tracks the non-private benchmark, especially at moderate to large sample sizes.  

\subsection{Two-Sided Hypothesis with Exponential Family} \label{sim: two-sided hypothesis}

In this section, we consider a two-sided hypothesis testing problem with $H_0$ given by a Logistic distribution with location $0$ and scale $1$ and $H_1$ given by a Logistic distribution with location $\theta_1$ and scale $1$, where $\theta_1 \in \{\pm 0.05, \pm 0.1, \pm 0.2\}$. We compare three methods: $\mu_{DP}(\underline{t};\ep)$, the private Kolmogorov--Smirnov (KS) test, and the private Cramér--von Mises (CvM) test.

Figure~\ref{fig:TwoSidedHT_Logistic_ep1} shows that, among the DP tests, our proposed $\mu_{DP}(\underline{t};\ep)$ attains the highest power across most alternatives and sample sizes. Its advantage becomes more evident as the sample size increases, and it closely tracks the power of the non-private \texttt{LLR} test as $n \ge 800$, particularly for intermediate values of $\theta_1$. 

\begin{figure}[t]
    \centering
    \includegraphics[width=0.75\linewidth]
    {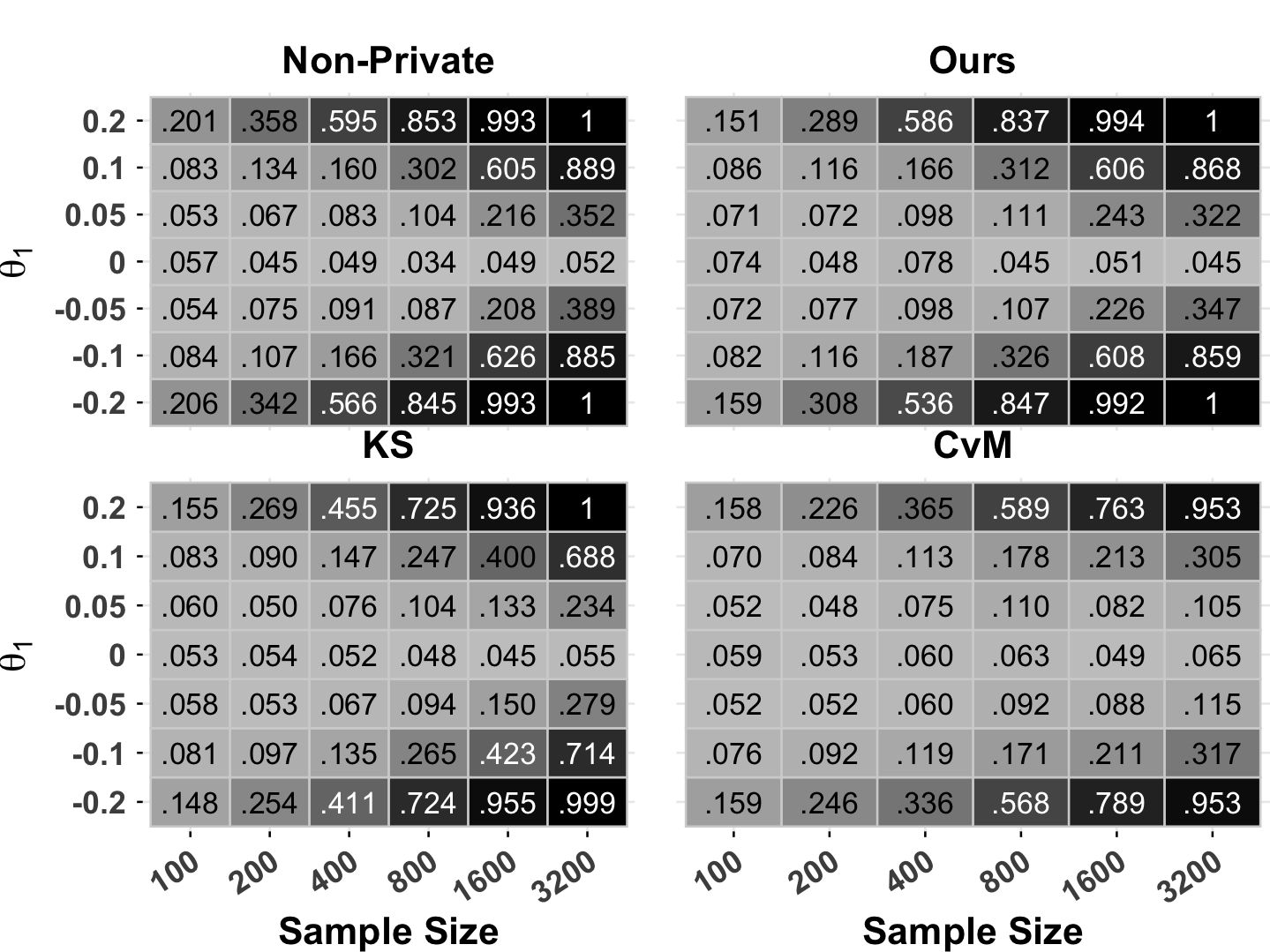}
    \caption{Two-sided hypothesis under Logistic data ($\epsilon=1$)}
    \label{fig:TwoSidedHT_Logistic_ep1}
\end{figure}

\section{Discussion}

We propose an adaptive clamping rule based on private quantiles for private mean estimation, optimally balancing DP noise and clamping bias by expanding and trimming the range to adapt to the data distribution. This design establishes population-level optimality---absent from prior private quantile-based approaches---recovers the correct constants, and yields near-optimal performance compared to non-private estimators. Applied to private hypothesis testing, our approach gives  tests for simple, and certain one- and two-sided hypotheses that achieve optimal asymptotic relative efficiency compared to the most powerful non-private tests. 

Despite these strengths, the framework has some notable limitations. The adaptive clamping rule is currently restricted to one-dimension. Extending it to higher dimensions would require identifying suitable projection strategies or norms that preserve efficiency under privacy constraints \citep{diakonikolas2019robust, hopkins2022efficientMean}.

While many applications focus on a single parameter of interest, nuisance parameters frequently arise in practice. Incorporating them in a principled manner under privacy constraints, for example via profile likelihood, conditional testing, or simulation-based calibration methods \citep{reid2003likelihood, andrews2016conditional, talts2018validating}, remains an important direction for future research.

Our approach relies on \texttt{GDP-Quant} as a preliminary step and therefore requires splitting the privacy budget, which can be less effective under very small privacy budgets or limited sample sizes (see Appendix~\ref{appendix: simulation} for empirical studies). Thus, identifying methods to optimize performance in these settings is a direction for future research.

For private tests, asymptotic relative efficiency is established, but we were unable to establish whether our method is rate-optimal compared to other DP methods, due to limitations of existing tools. Nevertheless, given the near-optimality proved for private mean estimation, we expect analogous results to hold, though a rigorous proof remains open.


\acks{} 
The project was supported in part by NSF award numbers SES 2610910 and SES 2150615.


\vskip 0.2in

\appendix

\section{More Simulation Results} \label{appendix: simulation}

\subsection{Complete Privacy Budgets for Private Mean Estimation}

The following two figures correspond to privacy budgets $\epsilon = 0.5$ and $\epsilon = 2$, completing the cases omitted from Section~\ref{sim: private mean estimation}.

\begin{figure}[htbp]
    \centering
    \includegraphics[width=0.75\linewidth]{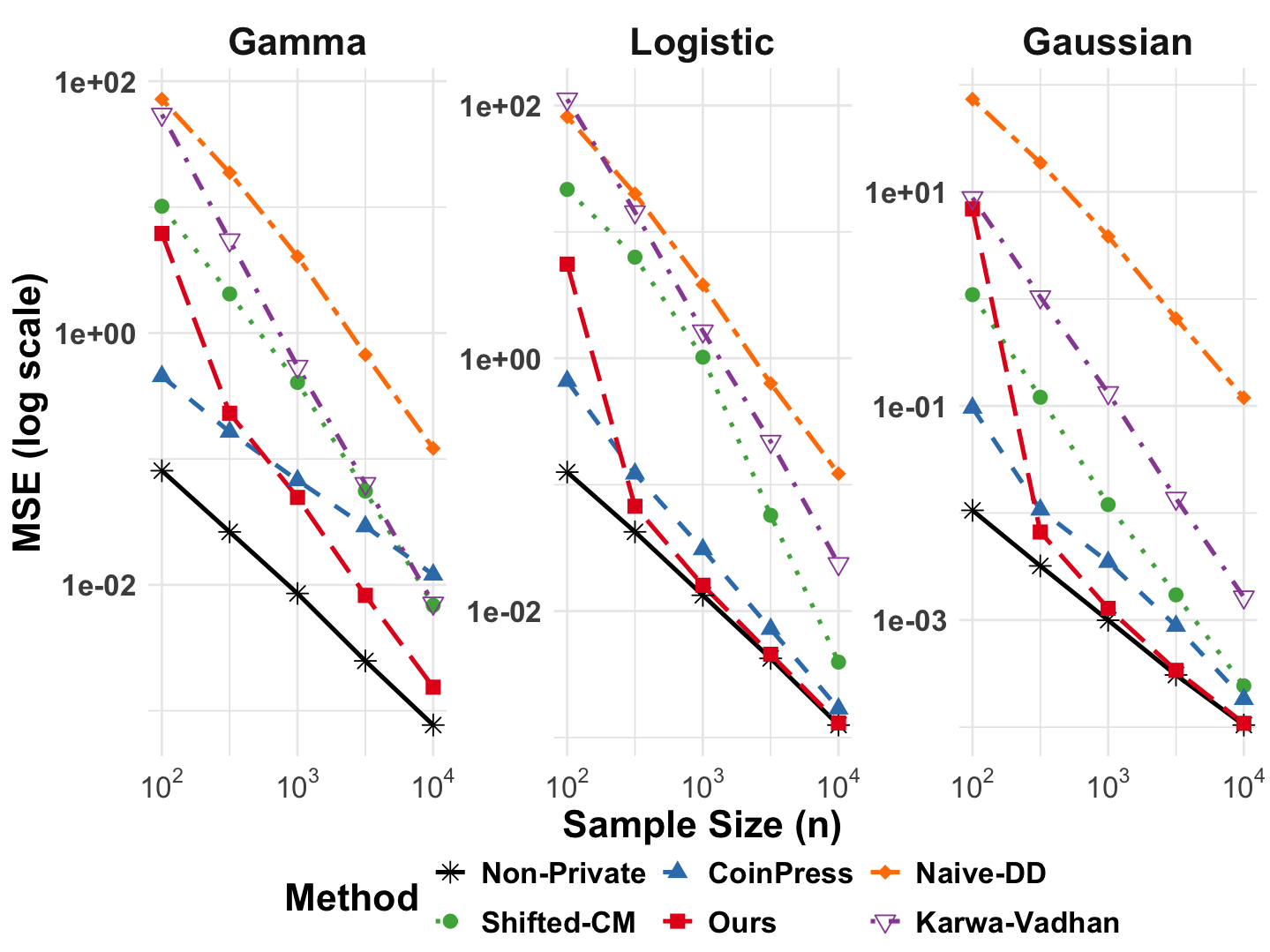}
    \caption{GDP mean estimation comparison ($\epsilon=0.5$)}
    \label{fig:GDP_MeanEst_Comparison_ep.5}
\end{figure}

Figure~\ref{fig:GDP_MeanEst_Comparison_ep.5} exhibits a trend similar to that in Figure~\ref{fig:GDP_MeanEst_Comparison_ep1}. However, when $\epsilon = 0.5$, the performance ranking is more pronounced: our method achieves the smallest mean squared error, while \texttt{Shifted-CM} incurs the largest.

\begin{figure}[htbp]
    \centering
    \includegraphics[width=0.75\linewidth]{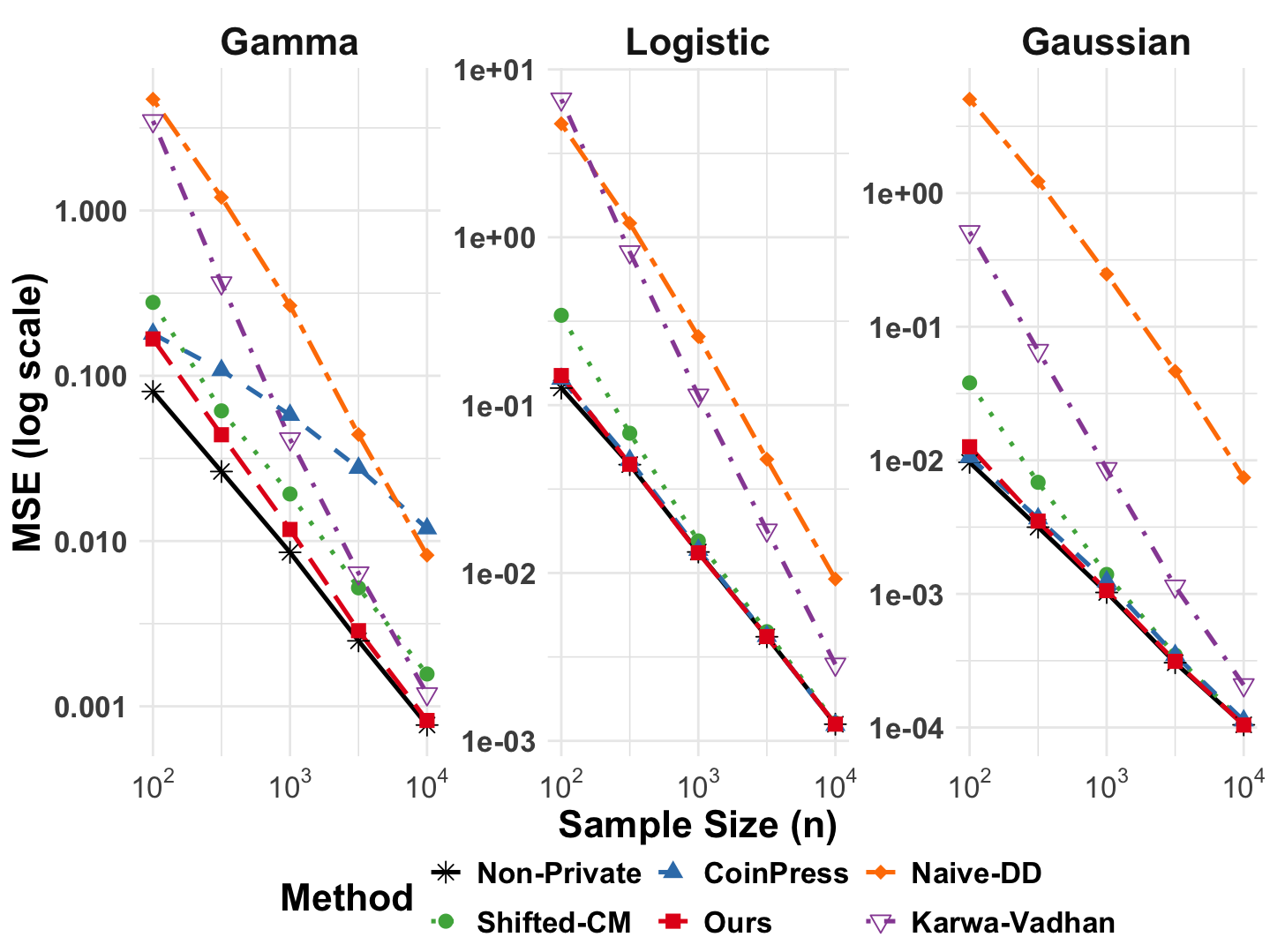}
    \caption{GDP mean estimation comparison ($\epsilon=2$)}
    \label{fig:GDP_MeanEst_Comparison_ep2}
\end{figure}

With $\epsilon = 2$, all four methods exhibit very similar slopes in Figure~\ref{fig:GDP_MeanEst_Comparison_ep2}. However, differences in the constant term again enable our method to achieve the smallest mean squared error.

\subsection{Complete Privacy Budgets for One-Sided Hypothesis}

The following two figures displays the cases with privacy budgets $\epsilon = 0.5$ and $\epsilon = 2$ that were omitted from Section~\ref{sim: one-sided hypothesis}.

\begin{figure}[htbp]
    \centering
    \includegraphics[width=0.75\linewidth]{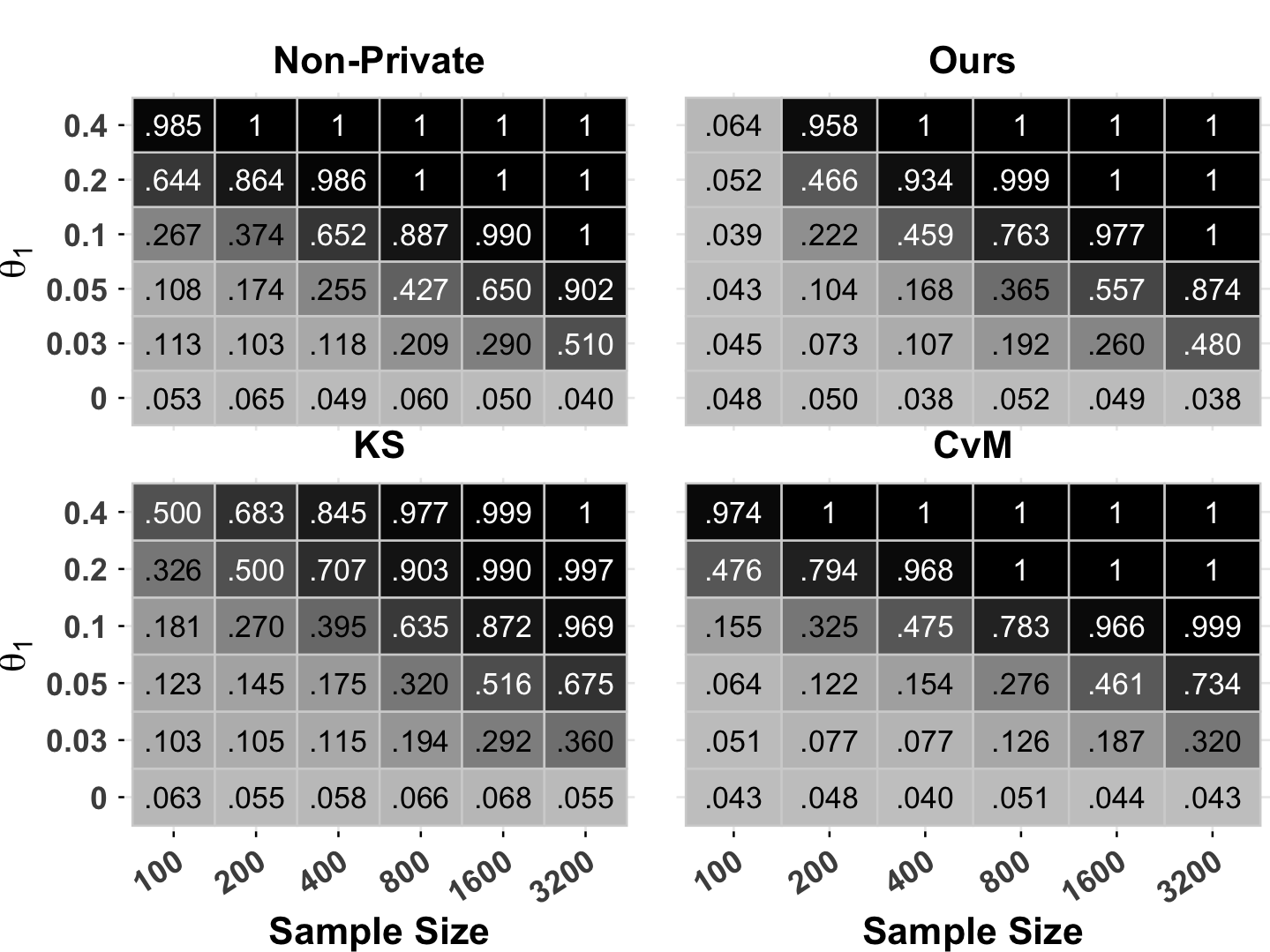}
    \caption{One-sided hypothesis under Gaussian data ($\epsilon=0.5$)}
    \label{fig:OneSidedHT_Cauchy_ep.5}
\end{figure}

When $\ep = 0.5$ and the sample size is 100, $\mu_{DP}(\underline{t};\ep)$ performs poorly in Figure~\ref{fig:OneSidedHT_Cauchy_ep.5}. This is because our method allocates part of the privacy budget to private quantile estimation. When the privacy budget is limited, the quantile estimator can be unstable, particularly for small sample sizes. However, $\mu_{DP}(\underline{t};\ep)$ improves rapidly and performs comparably to the other two methods at moderate sample sizes (200 and 400), and outperforms them when the sample size exceeds 800.

\begin{figure}[htbp]
    \centering
    \includegraphics[width=0.75\linewidth]{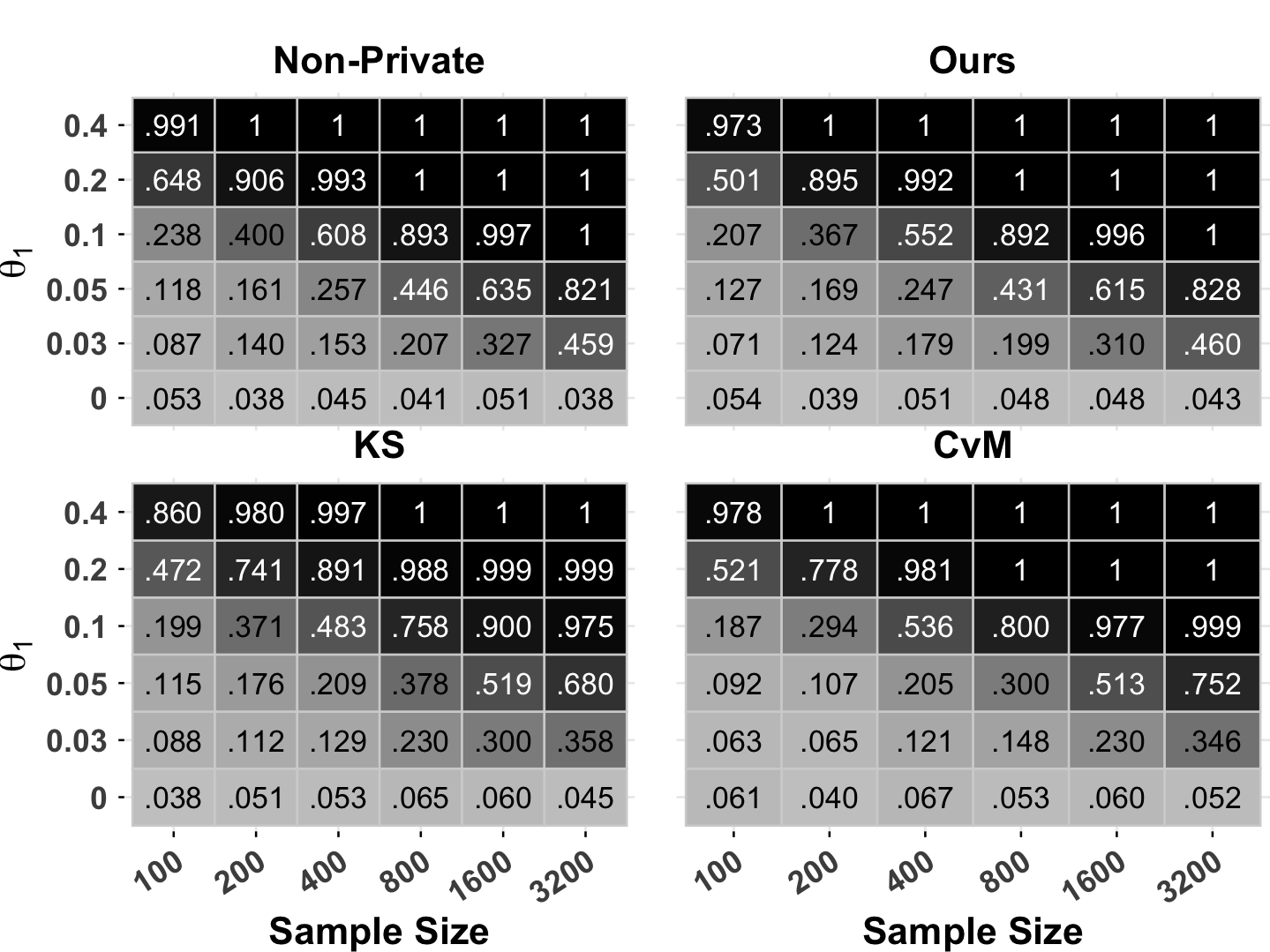}
    \caption{One-sided hypothesis under Gaussian data ($\epsilon=2$)}
    \label{fig:OneSidedHT_Cauchy_ep2}
\end{figure}

When $\epsilon = 2$, $\mu_{DP}(\underline{t};\ep)$ essentially dominates once the sample size exceeds 200 in Figure~\ref{fig:OneSidedHT_Cauchy_ep2}. More importantly, for alternatives with $\theta_1 \leq 2$---which correspond to more challenging cases for distinguishing the null from the alternative---$\mu_{DP}(\underline{t};\ep)$ is the most powerful among the three methods.

\subsection{Complete Privacy Budgets for Two-Sided Hypothesis}

The following two figures present the cases with privacy budgets $\epsilon = 0.5$ and $\epsilon = 2$ that were omitted from Section~\ref{sim: two-sided hypothesis}.

\begin{figure}[htbp]
    \centering
    \includegraphics[width=0.75\linewidth]{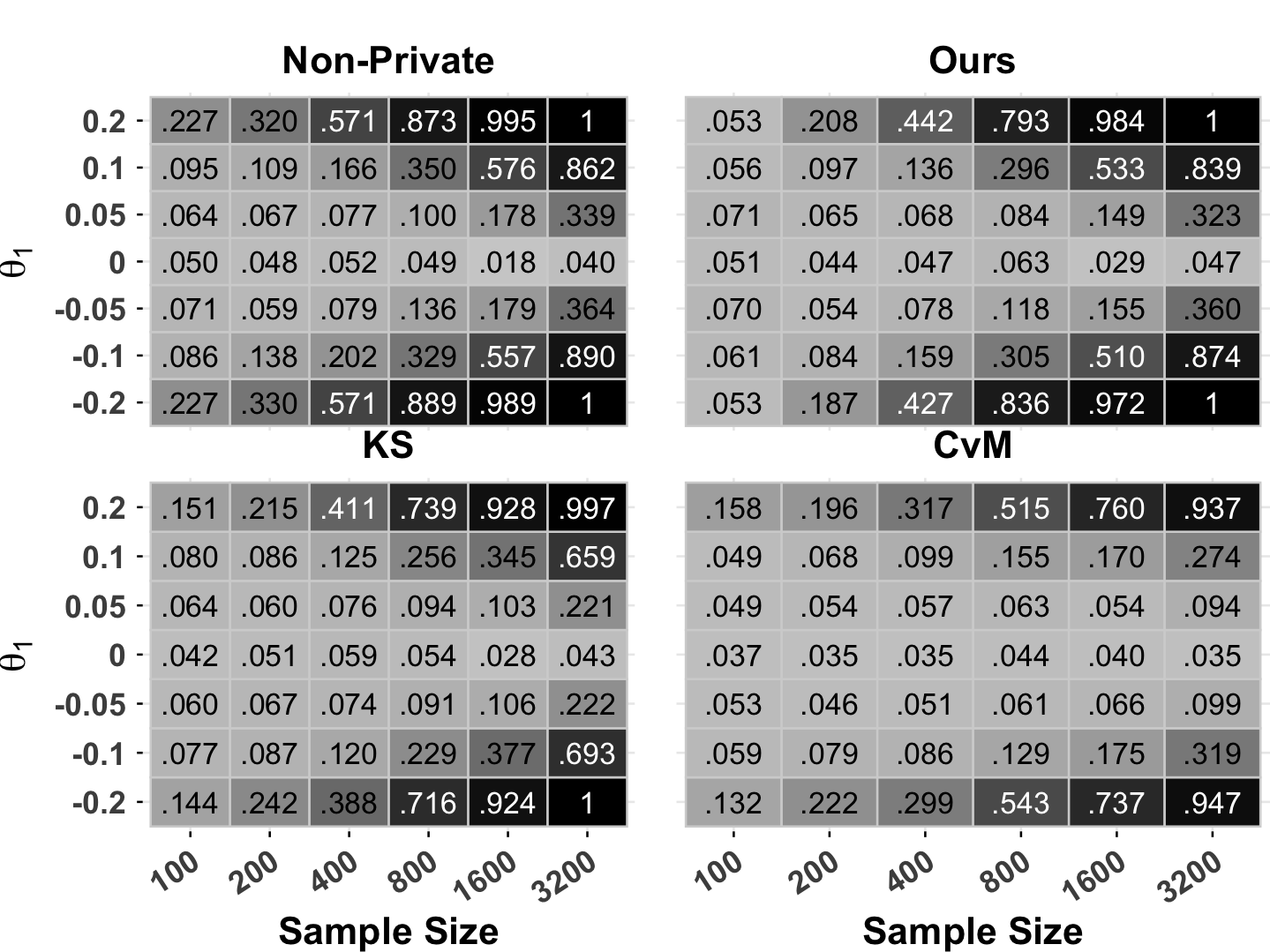}
    \caption{Two-sided hypothesis under Logistic data ($\epsilon=0.5$)}
    \label{fig:TwoSidedHT_Logistic_ep.5}
\end{figure}

While $\mu_{DP}(\underline{t};\ep)$ still suffers when the privacy budget is small ($\epsilon = 0.5$) and the sample size is 100, the other two methods also degrade under these conditions shown in Figure~\ref{fig:TwoSidedHT_Logistic_ep.5}. For alternatives with $-0.1 \leq \theta_1 \leq 0.1$, $\mu_{DP}(\underline{t};\ep)$ performs favorably relative to the the other two methods. This advantage becomes more pronounced across all values of $\theta_1$ for sample sizes exceeding 400.

\begin{figure}[htbp]
    \centering
    \includegraphics[width=0.75\linewidth]{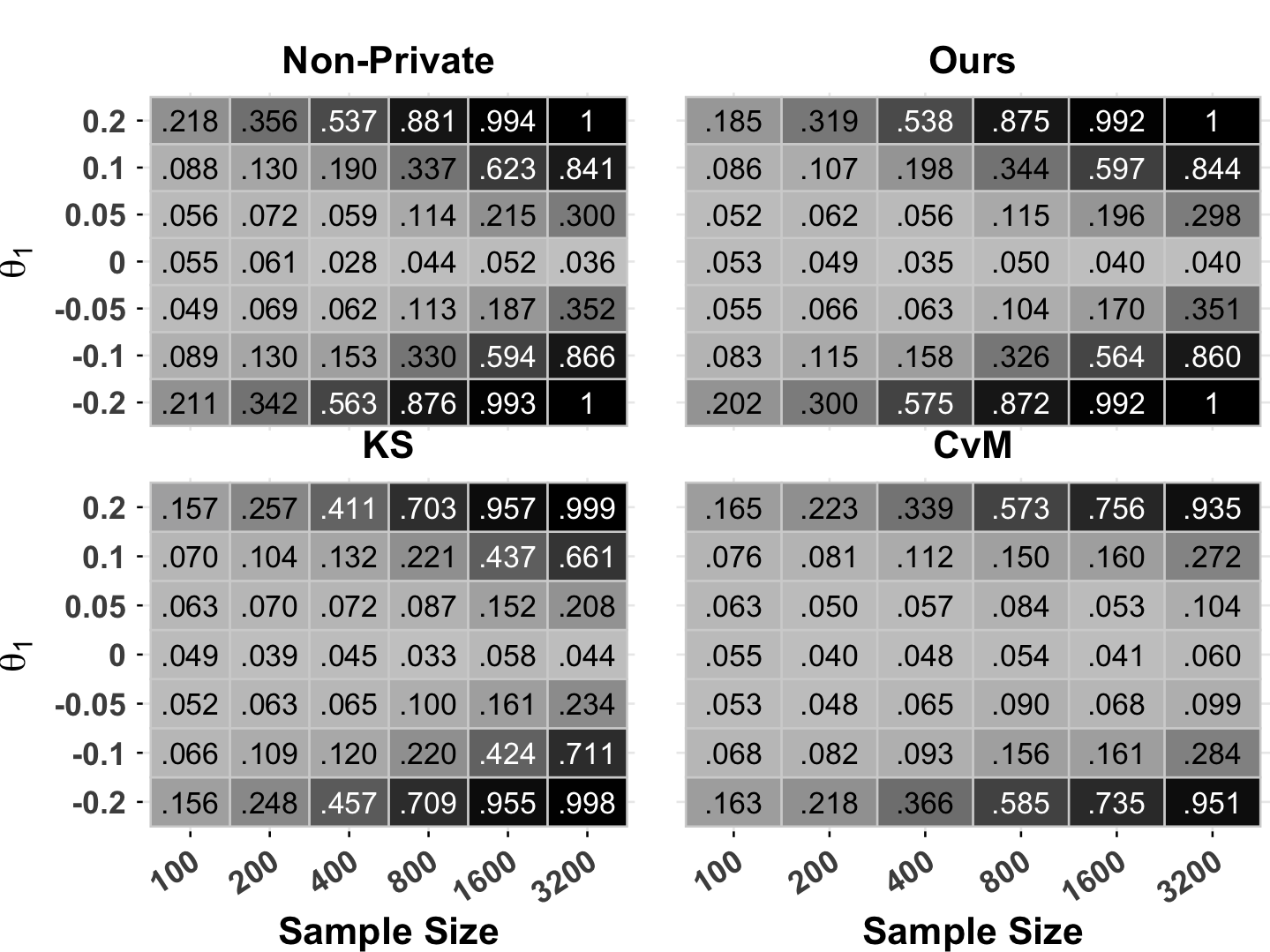}
    \caption{Two-sided hypothesis under Logistic data ($\epsilon=2$)}
    \label{fig:TwoSidedHT_Logistic_ep2}
\end{figure}

Consistent with the patterns observed in the previous sections, Figure~\ref{fig:TwoSidedHT_Logistic_ep2} demonstrates a clear advantage of $\mu_{DP}(\underline{t};\ep)$ over the other methods across all hypothesized values $\theta_1$ and sample sizes when $\epsilon = 2$.

\section{Proofs and Technical Details} \label{appendix: proof}

Before presenting the proofs in this paper, we first provide Example~\ref{eg: Huang's issues} as a counterexample to the claim in \citet{huang2021instance}.

\begin{example}[Noisy binary search with $\tau=1$] \label{eg: Huang's issues}
Let $D=\{1,\dots,10\}$, $n=10$, search range $[0,20]$, and target rank $m=3$ (true value $x_{(3)}=3$). 
Assume each noisy rank count satisfies $|\hat c - c|\le 1$ and equals $c-1$ in this example. 

\begin{itemize}
    \item \textbf{Iteration\ 1:} $(\text{left},\text{right},\text{mid})=(0,20,10)$; $c=10$, $\hat c=9$; since $\hat c>m$, set $\text{right}\leftarrow 10$.
    \item \textbf{Iteration\ 2:} $(0,10,5)$; $c=5$, $\hat c=4$; since $\hat c>m$, set $\text{right}\leftarrow 5$.
    \item \textbf{Iteration\ 3:} $(0,5,2)$; $c=2$, $\hat c=1$; since $\hat c<m$, set $\text{left}\leftarrow 3$.
    \item \textbf{Iteration\ 4:} $(3,5,4)$; $c=4$, $\hat c=3$; since $\hat c\le m$, set $\text{left}\leftarrow 5$.
    \item \textbf{Iteration\ 5:} $(5,5,5)$; loop ends and the algorithm outputs $v=5$.
\end{itemize}

The returned value has true rank $5$, yielding a rank error of $|5-3|=2$, which is $\tau + 1$; if there are $p$ data points taking the same value, then the rank error becomes $\tau + p$. These two facts were overlooked in \citet{huang2021instance}'s paper.
\end{example}

\subsection{Supplementary Material of Chapter~\ref{sec: private quantile estimation}: Additional Lemmas, Proofs, and Results}

Let the empirical CDF be $ F_n(\cdot): \mathbb{R} \to [0,1]$.  Recall that $w=(b-a)/2^T$ is discretization error and $[a,b]$ is partitioned into $2^T$ subintervals $B_k= [ a+ ( k- 1) w, a+ kw]$ with length $w$. For a dataset $\underline{x} =(x_1,\ldots,x_n)$, let 
\begin{align*}
N_k = \sum_{i=1}^n \mathbf{1}\{ x_i \in B_k \}   
\end{align*}
be the count of data points falling in each bin $B_k$. Recall that $q$ is the target quantile and $Z_1,\ldots,Z_T \overset{iid}{\sim}  N\left(0, \frac{T}{\ep^2}\right)$ Gaussian noise in Algorithm~\ref{alg: DP_quantile_selection}.

\begin{lemma} \label{lem: determinisitc property of GDP-Quant}
     Given a sample $\underline{x}=(x_1,\ldots,x_n)$ of size $n$. Let $l_t, m_t, r_t$ be the left, middle and right value for the $t$-th step in \texttt{GDP-Quant}. Let $D_t$ be a random variable such that $D_t=1$ if $q - F_n(m_t) >0$ and $D_t=-1$ if $q - F_n(m_t)<0$. Suppose
    \begin{enumerate}
        \item [(B)] $N_k\leq1$ for all $k=[2^T]$,
        \item [(C)] $-\infty < D_t Z_t < \tau$ for all $t \in [T]$.
    \end{enumerate} 
    Then the output of \texttt{GDP-Quant} has rank error less than or equal to $\tau + 1$.
\end{lemma}

\begin{proof}
    We prove by contradiction. Suppose that the output $m_T$ has rank error larger than $\tau + 1$. Without loss of generality, we consider $nF_n (m_T) < nq - (\tau+1)$. By condition (B) and the fact that $r_T = m_T + \frac{w}{2}$, we have $nF_n(r_T) \leq nF_n(m_T)+1$, which gives
    \begin{align} \label{eq: result from contradiction}
        nF_n(r_T) < nq-\tau.
    \end{align}
    Now, let $t^* = \min \{t: r_t = r_T\}$ be the first step $t$ the right value reaches $r_T$, then 
    \begin{align} \label{eq: result from t^*}
        nF_n(m_{t^*-1}) + Z_{t^*-1} \geq nq,
    \end{align}
    which can be seen from line 8 in \texttt{GDP-Quant}.
    Since $m_{t^*-1} = r_{t^*}$ and $q -F_n(r_{t^*}) >0$ from \eqref{eq: result from contradiction}, $q -F_n(m_{t^*-1}) >0$. By Assumption (C), $D_t=1$ and hence $D_{t^*-1} Z_{t^*-1} = Z_{t^*-1} <\tau$ . Thus, \eqref{eq: result from t^*} gives
    \begin{align*}
        nF_n(r_T) = nF_n(m_{t^*-1}) \geq nq - Z_{t^*-1} > nq - \tau,
    \end{align*}
    which contradicts to \eqref{eq: result from contradiction}.
\end{proof}

Lemma~\ref{lem: determinisitc property of GDP-Quant}  establishes a deterministic rank error bound of $\tau+1$ on events (B) and (C), where (B) implies that $T$ is sufficiently large that the search grid is fine enough to distinguish every data points from a continuous distribution and (C) requires errors introduced by the Gaussian noise to remain below $\tau$.  Although $Z_t$ represents additive noise, the lemma analyzes the algorithm on a restricted event where this noise satisfies condition (C). 

In contrast, the following Lemma~\ref{lem: GDP-Quant} considers the full randomness of $Z_t$, under which condition (C) holds only with high probability. Our analysis shows in order that, at termination, our algorithm may be ``off by one bin'' from the target rank error $\tau$, which translates into a rank error of at most $\tau+1$ because each bin contains at most one data point. Then, we conclude that condition (C) holds with high probability. Consequently, Lemma~\ref{lem: GDP-Quant} yields a probabilistic guarantee.

\dpQuantile*

\begin{proof}   
    We continue to use the notation defined in Lemma~\ref{lem: determinisitc property of GDP-Quant} for the proof. Recall that $D_t$ represents the underlying true direction in each step of \texttt{GDP-Quant}. That is, if $D_t =-1$, no matter how large of a positive $Z_t$ is added to $\#\{j: a \leq x_j \leq \text{mid}, \forall x_j \in \mathcal{D} \}$ at step $t$, \texttt{GDP-Quant} will move along the ground-truth direction. Therefore, we only need to bound $D_t Z_t$ instead of $|Z_t|$.
    
    To bound $D_t Z_t$, first recall that $Z_1,\ldots,Z_T \overset{iid}{\sim}  N\left(0, \sigma^2 = \frac{T}{\mu^2}\right)$, which are symmetric about zero. Let $Y_t = D_t Z_t$ and $\tau = \frac1\mu \sqrt{2T \log \frac{T}{\beta}}$. Then,
    \begin{align*} 
        P\left(\max_{1\leq i \leq T} Y_i > \tau \right) \leq  T \exp\left(\frac{-\tau^2}{2\sigma^2}\right) = \beta.
    \end{align*}
    Now, due to the discretization error, the output can have rank error $\tau + 1$.
\end{proof}

In Lemma~\ref{lem: GDP-Quant}, although $P\left(\max_{1\leq i \leq T} Y_i > \tau \right)$  can be tightly bounded using the product of Gaussian distribution functions, we adopt the maximal inequality because it provides a more straightforward interpretation of the resulting choice of $\tau$.

\subsection{Supplementary Material of Chapter~\ref{sec: private mean estimation}: Additional Lemmas, Proofs, and Results}

In this section, we develop a sequence of proofs, beginning with the simplest setting and building up into the setting required for Theorem~\ref{thm: GDP-MeanEst Utility}. Let
$$
S_{\underline{x}}=\{N_k\leq1 \text{ for all }k=1,\ldots,2^T\}
$$
be the event of all counts less than $1$ for all $2^T$ bins.

\begin{lemma} \label{lem: bounded domain assumed}
    Assume (A.1) wit a bounded domain $[a,b]$. Let number of steps $T = \lceil \log_2 [(b-a) n^\eta] \rceil$ with $\eta>0$. Then, probability $P(S_{\underline{x}})$  can be lower bounded by
    $$
    1-\binom{n}{2} \frac{M}{n^\eta}.
    $$
\end{lemma}

\begin{proof}
    Consider the probability of a data point falling in bin $B_k$,
    $$
    p_i \coloneqq P(X \in B_i) = \int_{B_i} f(x) dx,
    $$
    for all $1\leq k \leq 2^T$. Note that $\sum_{i=1}^{2^T} p_i = 1$ and the bin width $w = \frac{b-a}{2^T} \leq \frac{1}{ n^\eta}$.  Consider events $E_{ij}=\{ \text{points }i,j \text{ are in the same bin}\}$ and the number of colliding pairs can be defined as $E := \sum_{1\leq i<j \leq n} 1_{E_{ij}}$. Then, the probability of no collision is 
    \begin{align}
    P(S_{\underline{x}}) = P(E=0)= 1 - P(E \geq 1) 
    &\geq 1- \ex E \label{equ: markov} \\ 
    &= 1 - \sum_{1\leq i < j \leq n} P(E_{ij}) \nonumber \\
    &= 1 - \sum_{1\leq i < j \leq n} \sum_{k=1}^{2^T} P(X_i \in B_k) P(X_j \in B_k) \label{equ: independence} \\
    &= 1 - \sum_{1\leq i < j \leq n} \sum_{k=1}^{2^T} p_k^2 \nonumber \\
    &\geq 1-  \sum_{1\leq i < j \leq n}  \frac{(b-a)M}{2^T} \sum_{k=1}^{2^T} p_k  \label{equ: two bounds on pk}\\
    &\geq 1 - \binom{n}{2} \frac{M}{n^\eta}.\label{equ: 2^T is roundup from n^alpha}
    \end{align}

In the above, \eqref{equ: markov} is guaranteed by Markov's inequality, \eqref{equ: independence} holds as $X_i, X_j$ are independent for $i\neq j$, \eqref{equ: two bounds on pk} uses the fact that $p_k = \int_{B_k} f(x) dx \leq \frac{b-a}{2^T} M$, and \eqref{equ: 2^T is roundup from n^alpha} uses the fact that $T \geq \log (b-a) n^\eta$.
\end{proof}

Lemma~\ref{lem: bounded domain assumed} shows that a sufficiently large number of steps $T$ is required for the lower bound to converge to 1 as $\eta > 2$ and $n \to \infty$.

\begin{remark} 
    If the distribution in (A.1) is Uniform, then Lemma~\ref{lem: bounded domain assumed} reduces to the classic birthday problem. The probability of $n$ points lying in different bins are $P_{n,n^\eta} = \prod_{k=0}^{n-1}(1-\frac{k}{n^\eta})$, which converges to $1$ for $\eta>2$ as $n \to \infty$. (If $\eta = 2$, then the probability converges to $e^{-1/2}$.) Both cases can be seen from the Poisson approximation $P_{n,n^\eta} \approx e^{-\lambda}$ with rate $\lambda = \ex[E] = \binom{n}{2} / n^\eta$.
\end{remark}

\begin{lemma} \label{lem: colliding prob under conti dist}
    Assume (A.1). Let number of steps $T = \lceil \log_2 [(b-a) n^\eta] \rceil$ with $\eta>0$. Let search range $[a,b]$ be bounded and $x_i^c = \min \{\max\{x_i, a\}, b\}$, then probability $P(S_{\underline{x^c}})$ can be lower bounded by
    $$
    1-\binom{n}{2} \left( \frac{M}{n^\eta} +p_a^2+p_b^2 \right),
    $$
    where $p_{a} = P(X_1 < a+w)$ and $p_{b} = P(X_1 > b-w)$.
\end{lemma}

\begin{proof}
    Since $X_1,\ldots, X_n$ are i.i.d. with left tail probability $p_{a} = P(X_1 < a+w)$, $C_a \sim \text{Binomial}(n,p_a)$, where $C_a = |\{i: X_i < a+w\}|$. As the probability of the event of $\{C_a\geq 2\}$:
    \begin{align} \label{eq: probability of Ca}
    P(C_a\geq2) =  1 - P(C_a=0) - P(C_a=1) = 1-(1-p_a)^n - n p_a (1-p_a)^{n-1} 
    \end{align}
    is hard to factorize, we instead consider events $A_{ij}=\{ \text{points } i,j \text{ are below } a+w\}$ and the number of colliding pairs can be defined as $A := \sum_{1\leq i<j \leq n} 1_{A_{ij}}$. Then, \eqref{eq: probability of Ca} can be upper bounded as follows:
    \begin{align}
        P(C_a\geq2)
        &= P(A \geq 1) \label{equ: more than one point on the tail means at least one colliding pair}\\
        &\leq \ex A \label{equ: markov1}\\
        &= \sum_{1\leq i<j \leq n} P(A_{ij}) \nonumber\\
        &= \sum_{1\leq i<j \leq n} P(X_i< a+w)  P(X_j< a+w) \label{equ: independence1}\\
        &= \binom{n}{2} p_a^2. \nonumber
    \end{align}
    The first equality \eqref{equ: more than one point on the tail means at least one colliding pair} holds because having no less than two points on the tail is equivalent to having at least one colliding pair over the same tail region,  \eqref{equ: markov1} is guaranteed by the Markov's inequality, and \eqref{equ: independence1} holds as $X_i, X_j$ are independent for $i\neq j$.
    
    Similarly, we can derive a similar result for event $B := \sum_{1\leq i<j \leq n} 1_{B_{ij}}$, where $B_{ij}=\{ \text{points } i,j \text{ are above } b-w\}$ and $p_b = P(X>b-w)$. Then, the probability of no collision within $[a,b]$ and no more than two points outside $[a+w,b-w]$ is 
    \begin{align*}
    1-P(E\geq1,A\geq1,B\geq1) 
    &\geq 1 - P(E\geq1) - P(A\geq1) - P(B\geq1) \\
    &= 1 - \binom{n}{2} \left( \frac{M}{n^\eta} + p_a^2+p_b^2 \right).
    \end{align*}    
\end{proof}

Lemma~\ref{lem: colliding prob under conti dist} relaxes the bounded-domain assumption by enforcing a bounded search range through clamping the data, and therefore yields a slightly smaller lower bound than Lemma~\ref{lem: bounded domain assumed}.

\begin{lemma} \label{lem: subexpoenntial assumed}
    Assume (A.1) and (A.2). Let number of steps $T = \lceil \log_2 [(b-a) n^\eta] \rceil$ with $\eta>0$. Let search range $[a,b] = [\mu - \frac{s\eta}{2}\log n, \mu + \frac{s\eta}{2}\log n]$. Let $x_i^c = \min \{\max\{x_i, a\}, b\}$, then probability of  $P(S_{\underline{x^c}})$ can be lower bounded by
    $$
    1 - \frac{1}{n^{\eta-2}} \left( \frac{M}{2}+1 \right).
    $$
\end{lemma}

\begin{proof}
    By the definition of a subexponential distribution, for all $c>0$ 
    $$
    P(|X_1 -\mu|\geq c) = P(X_1 -\mu \leq -c, X_1 -\mu \geq c) \leq 2\exp(-c/s)
    $$ 
    for the scale parameter $s$. Let $\exp(-c/s) = n^{-\eta/2}$ to match the tail probability $p_a, \, p_b$ defined in Lemma~\ref{lem: colliding prob under conti dist}. We solve for $c(n)= \frac{s \eta}{2}\log n$, which is a function of $n$. Then, to have $p_{a(n)} = p_{b(n)} = n^{-\eta/2}$, the search range $[a(n), b(n)]$ should satisfy $a(n)+w = \mu-c(n)$ and $b(n)-w = \mu+c(n)$. Now, by the result of Lemma~\ref{lem: colliding prob under conti dist}, 
    \begin{align*}
        1-P(E\geq1,A\geq1,B\geq1)
        &\geq  1 - \binom{n}{2} \left( \frac{M}{n^\eta} + p_{a(n)}^2+p_{b(n)}^2 \right)  \\
        &= 1 - \binom{n}{2} \frac{M+2}{n^\eta} \\
        &\geq 1 - \frac{1}{n^{\eta-2}} \left( \frac{M}{2}+1 \right).
    \end{align*}
\end{proof}

Starting from Lemma~\ref{lem: subexpoenntial assumed}, Assumption (A.2) will be included in all subsequent lemma and theorem statements. To accommodate the subexponential data, the search range must expand logarithmically to preserve the lower bound established in Lemma~\ref{lem: colliding prob under conti dist}.

\begin{lemma} \label{lem: GDP-Quant small rank error by sufficient bins with high probability}
    Assume (A.1) and (A.2). Let number of steps $T = \lceil \log_2 [(b-a) n^\eta] \rceil$ with $\eta>0$. Let search range $[a,b] = [\mu - \frac{s\eta}{2}\log n, \mu + \frac{s\eta}{2}\log n]$. Let $\epsilon>0, T \geq 1$ and $\beta \in (0,1)$ such that $\tau = \frac1\ep \sqrt{2T \log \frac{T}{\beta}}$. Let $x_i^c = \min \{\max\{x_i, a\}, b\}$, then \texttt{GDP-Quant} satisfies $\epsilon$-GDP and returns a quantile with rank error less than $\tau + 1$ and  with probability at least 
    $$
    (1-\beta)\left( 1 - \frac{1}{n^{\eta-2}} \left( \frac{M}{2}+1 \right) \right).
    $$
\end{lemma}

\begin{proof}
    Since the two sources of randomness—the subexponential sample distribution and the additive Gaussian noise—are independent, the result follows by applying Lemma~\ref{lem: GDP-Quant} and Lemma~\ref{lem: subexpoenntial assumed}.
\end{proof}


Lemma~\ref{lem: GDP-Quant small rank error by sufficient bins with high probability} explains why in Algorithm~\ref{alg: DP_instance_mean_estimation} the two target quantiles are chosen as
\begin{align*}
    q_l = \frac{\tau + 2}{n}
\quad \text{and} \quad
q_u = 1 - \frac{\tau + 1}{n}.
\end{align*}
These choices correspond to the $(\tau+1)$-st and $(n-\tau-1)$-st order statistics, respectively. 

\dpInstMeanEstUtility*

\begin{proof}
    Fix $\gamma \in (0,1)$. For constants $r^*>0$, $\tau \ge 0$, and $z^*>0$, define
    \begin{align*}
    E_1(r^*) 
    &:= \left\{ \underline{X} : \text{the range of } \underline{X} \le r^* \right\}, \\
    E_{2_l}(\tau) 
    &:= \left\{ (\underline{X},\underline{z}_l) :
    \texttt{GDP-Quant}(\underline{X},a,b,T,q_l,\ep_q) \text{ has rank error} \leq \tau+1 \right\}, \\
    E_{2_u}(\tau) 
    &:= \left\{ (\underline{X},\underline{z}_u):
    \texttt{GDP-Quant}(\underline{X},a,b,T,q_u,\ep_q) \text{ has rank error} \leq \tau+1 \right\}, \\
    E_3(r^*,z^*) 
    &:= \left\{ z_m :
    |z_m| \le \frac{r^* z^*}{n\epsilon_m} \right\},
    \end{align*}
    where $\underline{X}\sim F^n$, $\underline{z}_l, \underline{z}_u \overset{iid}{\sim} N\left(0, \tfrac{T}{\ep_q^2} I_T \right)$ and $z_m \sim N\!\left(0,\; \left(\tfrac{r(\underline{X})}{n\epsilon_m}\right)^2 \right)$. Define the joint event
    \begin{align*}
        G(r^*,\tau,z^*) := E_1(r^*) \cap E_{2_l}(\tau) \cap E_{2_u}(\tau) \cap E_3(r^*,z^*).
    \end{align*}
    We will choose $r^*,\tau,z^*$ later to ensure $P\bigl(G(r^*,\tau,z^*)\bigr)\ge 1-\gamma$.

    To make the probability that $G$ happens at least $1-\gamma$, we set $P(E_i)=1-\frac{\gamma}{4}$ and determine the corresponding values of $r^*, \tau, z^*$. For $E_1$, consider
    \begin{align*}
     P\left(\max_{1\le i \le n} |X_i - \mu| > c \right)  
     &= P\left(\bigcup_{1\le i \le n} \{|X_i - \mu| > c\} \right)  \\
     &\le n P(|X_1 - \mu| > c) \\
     &\le 2n \exp(-c/s) \\
     &\overset{\text{set}}{\le} \frac{\gamma}{4}.
    \end{align*}
    Therefore, $c = s \log (\frac{8n}{\gamma})$ and hence 
    \begin{align*}
        r^* = 2s \log \left(\frac{8n}{\gamma}\right).    
    \end{align*}
    
    For $E_{2_l}$, according to Lemma~\ref{lem: GDP-Quant small rank error by sufficient bins with high probability}, we let $\beta=\frac{\gamma}{8}$. Then, there exists large $n$ such that $n^{2-\eta} \left( \frac{M}{2}+1 \right) \leq \frac{\gamma}{8}$, so 
    \begin{align*}
        \tau = \frac{1}{\ep_q} \sqrt{2T \log \frac{8T}{\gamma}}.
    \end{align*}
    
    Note that to establish a high-probability statement, it suffices to verify that a finite $\eta$ can be determined, which is indeed possible. Since our goal is an $O_p$ result and $\eta > 2$ is fixed, the inequality to hold for sufficiently large $n$. Similarly, we can replicate for $E_{2_u}$ as well. 

    For $E_3$, set 
    \begin{align*}
        z^* = \Phi^{-1}(1-\frac{\gamma}{8}),
    \end{align*}
    then we have $|z_m| \leq \frac{r^*}{n\ep_m} z^*$ with probability $1-\gamma/4$.

    Now, for any given $\gamma \in (0,1)$ and sufficiently large $n$, the event $G$ that has probability $1-\gamma$ satisfies
    $$
    |\mu_{DP} - \bar{X}| \leq |\mu_{DP} - \bar{X_c}| + |\bar{X_c} - \bar{X}| \leq \frac{r^* z^*}{n \ep_m} + \frac{r^* (4\tau+4)}{n},
    $$
    since $\bar{X} - \frac{r^* (4\tau+4)}{n} \leq \bar{X_c} \leq \bar{X} + \frac{r^* (4\tau+4)}{n}$ from (1) and (2).

    To claim that for all $\gamma \in (0,1)$, there exists a constant $\Gamma(\gamma)$ such that for all $n>N(\gamma)$,
    $$
    P\left(\frac{|\mu_{DP} - \bar{X}|}{m(n)} \leq \Gamma \right) \geq 1- \gamma,
    $$
    where $\mathcal{M}$ and $\underline{X}$ denote the randomness from the mechanism (Algorithm~\ref{alg: DP_instance_mean_estimation}) and the data, respectively. Thus, we need to factorize the upper bound of $\frac{r^* z^*}{n \ep_m} + \frac{r^* (4\tau+4)}{n}$ into $\Gamma m(n)$ for sufficiently large sample sizes.

    Since $r^* = 2s \log (\frac{8n}{\gamma})$, then for any fixed $\gamma \in (0,1)$, there exists $N_1(\gamma)$ such that for all $n > N_1$,
    \begin{align*} 
        r^*
        = 2s \log n \left( 1+\frac{\log(8/\gamma)}{\log n} \right) 
        \leq A_1(\gamma) s \log n
    \end{align*}
    with $A_1(\gamma) = \left( 1+\frac{\log(8/\gamma)}{\log N_1(\gamma)} \right)$ and 
    \begin{align*}
        \frac{1}{\ep_m} 
        = \frac{1}{\ep} \left( 1-\frac{2}{(\log n)^{2k}} \right)^{-1/2} 
        \leq B_1 (\gamma)  \frac{1}{\ep}
    \end{align*}
    with $B_1(\gamma) = \left( 1-\frac{2}{(\log N_1(\gamma))^{2k}} \right)^{-1/2}$. These follow from the fact that $\left( 1+\frac{\log(8/\gamma)}{\log n} \right)$ and $\left( 1-\frac{2}{(\log n)^{2k}} \right)^{-1/2}$ are both decreasing in $n$.  Thus, 
    \begin{align} \label{eq: bound on noise}
        |\mu_{DP} - \bar{X}_c| \leq \frac{r^* z^*}{n \ep_m} \leq \frac{A_1(\gamma) B_1(\gamma) \phi(\gamma) s}{\ep} \frac{\log n}{n},
    \end{align}
    where $\phi(\gamma) = \Phi^{-1}(1-\frac{\gamma}{8})$.
    
    Since $\tau = \frac{1}{\ep_q} \sqrt{2T(n) \log \frac{8T(n)}{\gamma}}$, then for any fixed $\gamma \in (0,1)$, there exists $N_2(\gamma)$ such that for all $n > N_2$,
    \begin{align*}
        \tau 
        = \frac{\sqrt{2T(n) \log T(n)}}{\ep_q} \sqrt{ 1+ \frac{\log \frac{8}{\gamma}}{\log T(n)}} 
        \leq \frac{(\log n)^k \sqrt{T(n) \log T(n)}}{\ep} A_2(\gamma)
    \end{align*}
    with $A_2(\gamma) = \sqrt2\sqrt{1+\frac{\log \frac{8}{\gamma}}{\log T(N_2(\gamma))}}$. This follows from the fact that $\sqrt{ 1+ \frac{\log \frac{8}{\gamma}}{\log T(n)}}$ is decreasing in $n$. Thus, 
    \begin{align} \label{eq: bound on clamping bias}
        |\bar{X}_c -  \bar{X}| \leq \frac{4r^*(1+\tau)}{n} = 4s A_1(\gamma) \frac{\log n}{n} + \frac{4s}{\ep} A_1(\gamma) A_2(\gamma) \frac{(\log n)^{1+k} \sqrt{T(n) \log T(n)}}{n}.
    \end{align}

    Together, for any $\gamma \in (0,1)$, there exists $N_3(\gamma) \geq  \max \{N_1(\gamma), N_2(\gamma)\}$ such that for all $n > N_3$,
    \begin{align}
        |\mu_{DP} - \bar{X}| 
        &\leq \left(\frac{B_1(\gamma) \phi(\gamma)}{\ep} + 4\right)s A_1(\gamma) \frac{\log n}{n} + \frac{4s}{\ep} A_1(\gamma) A_2(\gamma) \frac{(\log n)^{1+k} \sqrt{T(n) \log T(n)}}{n} \label{equ: using bound on noise and clamping bias}\\
        &\leq s A_1(\gamma) \frac{(\log n)^{1+k} \sqrt{T(n) \log T(n)}}{n} \left(\frac{B_1(\gamma)\phi(\gamma)/\ep + 4}{(\log n)^k \sqrt{T(n)\log T(n)}}+\frac{4}{\ep}A_2(\gamma)\right) \nonumber \\
        &\leq \frac{(\log n)^{1+k} \sqrt{T(n) \log T(n)}}{n} \frac{s}{\ep} A_1(\gamma) \left[ (B_1(\gamma)\phi(\gamma) + 4\ep)A_3(\gamma)+4 A_2(\gamma)\right] \label{equ: bound from n larger than N3} \\
        &= m(n) \cdot  \Gamma(\gamma) \nonumber,
    \end{align}
    where $A_3(\gamma)= \frac{1}{\log N_3(\gamma) \sqrt{T(N_3(\gamma))\log T(N_3(\gamma))}}$ and $m(n)=\frac{(\log n)^{1+k} \sqrt{T(n) \log T(n)}}{n}$ such that $\Gamma(\gamma) = \frac{s}{\ep} A_1(\gamma) \left[ (B_1(\gamma)\phi(\gamma) + 4\ep)A_3(\gamma)+4 A_2(\gamma)\right]$. \eqref{equ: using bound on noise and clamping bias} follows from \eqref{eq: bound on noise} and \eqref{eq: bound on clamping bias}. \eqref{equ: bound from n larger than N3} holds as $n > N_3$. Thus,
    $$
    |\mu_{DP} - \bar{X}|  = O_p\left(\frac{s (\log n)^{1+k} \sqrt{T(n) \log T(n)}}{\ep n}\right),
    $$
    where $T(n) = \Theta(p \log_2 (\log n) + \eta \log_2 n)$. This is because 
    $$
    T(n) \leq 1+ \log_2 [(b(n)-a(n)) n^\eta] = 1+ \log_2 \left[ 2v(\log n)^p n^\eta \right] = 2 + \log v + p \log_2 (\log n) + \eta \log_2 n.
    $$
    In addition, we can now write the central limit theorem result of $\bar{X}$ as $\sqrt{n}(\bar{X}-\theta)\xrightarrow{d}N(0, \sigma^2)$ for some $\theta \in \mathbb{R}$ and $\sigma^2 > 0$. Since $\mu_{DP} = \bar{X} + o_p(\frac{1}{\sqrt{n}})$, 
    $$
    \sqrt{n}(\mu_{DP} - \theta) = \sqrt{n}(\bar{X}-\theta) + o_p(1) \xrightarrow{d}N(0, \sigma^2).
    $$
\end{proof}

Using Lemma~\ref{lem: GDP-Quant small rank error by sufficient bins with high probability}, Theorem~\ref{thm: GDP-MeanEst Utility} establishes an $O_p$ bound on the absolute error between the private and non-private means. Proposition~\ref{prop: GDP minimax mean estimation lower bound} below then shows that this bound matches the GDP minimax lower bound for mean estimation.

\begin{proposition}  \label{prop: GDP minimax mean estimation lower bound}
    Let $\mathcal{P}_{\text{subexp}}$ be the family of distributions that satisfy (A.1) and (A.2) and $\mathcal{Q}_\epsilon^{\text{GDP}}$ be the family of privacy mechanisms that satisfy $\epsilon$-GDP. Define the minimax risk under absolute error loss by
    $$
    \mathfrak{M}_n\!\left(\mu(\mathcal{P}_{\mathrm{subexp}}), \mathcal{Q}_{\epsilon}^{\mathrm{GDP}}, |\cdot|\right)
    := \inf_{\mathcal{M} \in \mathcal{Q}_{\epsilon}^{\mathrm{GDP}}}
    \sup_{P \in \mathcal{P}_{\mathrm{subexp}}}
    \mathbb{E}_{P,\mathcal{M}}\!\left[\,|\widehat{\mu} - \mu(P)|\,\right],
    $$
    where the infimum is taken over all $\epsilon$-GDP mechanisms and $\widehat{\mu}$ denotes the resulting estimator. Consider the mean estimation problem over $\mathcal{P}_{\text{subexp}}$ under $\epsilon$-GDP, then
    \begin{align*}
        \mathfrak{M}_n(\mu(\mathcal{P_{\text{subexp}}}), \mathcal{Q}_\epsilon^{\text{GDP}}, |\cdot|) = \widetilde{\Omega}_p \left( \frac{\sigma}{\sqrt{n}} + \frac{\sigma}{n\ep} \right).
    \end{align*}
\end{proposition}

\begin{proof}
    We follow \citet{barber2014privacy}'s notation for the proof and apply Le Cam's method.  First, fix $\delta > 0$ (to be decided later) and define 
    \begin{align*}
        P_0 \sim N(\mu_0 = -\delta, \sigma_0^2 = \sigma^2),\ P_1 \sim N(\mu_0 = \delta, \sigma_1^2 = \sigma^2) 
    \end{align*}
    to be normally distributed, which are indeed subexponential. The total variation between $P_0$ and $P_1$ is,
    \begin{align} \label{eq: upper bound of TV}
        \| P_0 - P_1 \|_{TV} \leq \left( \frac12 \mathrm{KL}(P_0 \| P_1) \right)^{1/2} = \left( \frac12 \left(\log\frac{\sigma_1}{\sigma_2} + \frac{\sigma_0^2 + (\mu_0 - \mu_1)^2}{2 \sigma_1^2} - \frac{1}{2}\right) \right)^{1/2} = \frac{\delta}{\sigma}.
    \end{align}
    
    Then, the minimax lower bound for any $\alpha$-total variation mechanism $Q \in \mathcal{Q}_\alpha^{\text{TV}}$ under absolute error loss:
    \begin{align*}
        m(\mu(\mathcal{P_{\text{subexp}}}), Q, |\cdot|) \geq \frac{1}{2} \left| \frac{\mu_0 - \mu_1}{2} \right| \left( 1 - 2n\alpha \|P_0 - P_1\|_{TV} \right) = \frac{\delta}{2} \left(1 - 2n\alpha \frac{\delta}{\sigma} \right)
    \end{align*}
    where the inequality is given by the Le Cam's Lemma \citep{barber2014privacy} and \eqref{eq: upper bound of TV}. Choosing $\delta = \sigma/(4n\alpha)$, we substitute to have that 
    \begin{align*}
        m(\mu(\mathcal{P_{\text{subexp}}}), Q, |\cdot|) \geq \frac{\sigma}{16 n \alpha}.
    \end{align*}
    To apply this bound to the GDP setting, we first note that a $\ep$-GDP mechanism satisfies $(0,\ep/2)$-DP: 
    \[\mathrm{TV}(N(0,1),N(\ep,1))\leq \left(\frac 12 \mathrm{KL}(N(0,1)||N(\ep,1))\right)^{1/2}=\left(\frac 12 \left[\frac{\ep^2}{2}\right]\right)^{1/2}=\frac{\ep}{2},\]
    where we used Pinsker's inequality and standard calculations of the KL divergence. 
    Substituting yields,
    \begin{align*}
        m(\mu(\mathcal{P_{\text{subexp}}}), Q, |\cdot|) \geq \frac{\sigma}{8 n \ep}.
    \end{align*}
    Since the non-private bound is standard, we have the sum of two as the minimax lower bound of such mean estimation problem.
\end{proof}

\subsection{Supplementary Material of Chapter~\ref{sec: near-optimal private tests}: Additional Lemmas, Proofs, and Results}

In this section, we develop a sequence of proofs, beginning with simple hypotheses and extending to the one- and two-sided cases.

\AREofncLLR*

\begin{proof}
For $H_0:P$ versus $H_1:Q$, the classical likelihood ratio test rejects for large 
$\sum_{i=1}^n \log(q(X_i)/p(X_i))=-\sum_{i=1}^n \ell(X_i)$.  
By the Chernoff--Stein Lemma in chapter~12 of \citet{cover2001elements}, its type-II error under $Q$ decays with exponent $\mathrm{KL}(Q\|P)$; hence its Bahadur slope is $2\mathrm{KL}(Q\|P)$.

Suppose only $Y_i=g(X_i)$ are observed, then under $H_0$ and $H_1$ the sample is i.i.d.\ from $P_Y$ and $Q_Y$. Chernoff--Stein again gives that any test based on $(Y_1,\dots,Y_n)$ (with any $n$-independent additive noise) has Bahadur slope at most $2\mathrm{KL}(Q_Y\|P_Y)$.

To show $\mathrm{KL}(Q_Y\|P_Y) < \mathrm{KL}(Q\|P)$, let $P_{X,Y}$ and $Q_{X,Y}$ denote the joint laws of $(X,Y)$ under $P$ and $Q$, with $Y=g(X)$ deterministic.  
Since $\mathrm{KL}(Q\|P)=D(Q_{X,Y}\|P_{X,Y})$, the KL chain rule yields
$$
\mathrm{KL}(Q\|P)=\mathrm{KL}(Q_Y\|P_Y)+\ex_Q\!\big[ D(Q_{X\mid Y}\,\|\,P_{X\mid Y}) \big].
$$
Nontrivial clamping implies that for some $y\in\{a,b\}$ we have $P(Y=y),Q(Y=y)>0$ and $g$ is not injective on $\{x:g(x)=y\}$. Without loss of generality, we discuss the case of $y=b$ with $S_b := \{x: \ell(x) \ge b\}$. Since $\ell(x)$ is not constant on $S_b$, $Q_{X\mid Y=b}\neq P_{X\mid Y=b}$. The same argument works with $Y=a$. Together, we have
$$
D(Q_{X\mid Y=y}\,\|\,P_{X\mid Y=y})>0,
$$
which gives
$$
\ex_Q[ D(Q_{X\mid Y}\,\|\,P_{X\mid Y}) ]>0,
$$
and therefore $\mathrm{KL}(Q_Y\|P_Y) < \mathrm{KL}(Q\|P)$. It follows that any $Y$-based test has Bahadur slope strictly smaller than $2\mathrm{KL}(Q\|P)$, and hence Bahadur ARE $<1$ relative to the classical LLR test.
\end{proof}

Proposition~\ref{prop: AREofncLLR} shows that any non-trivial fixed clamping incurs a loss in effective sample size, as is the case for \texttt{ncLLR}. The following example illustrates this phenomenon.

\begin{example}
In the Gaussian setting $H_0:P=N(0,1)$ and $H_1:Q=N(\delta,1)$ with $\delta\neq 0$, the log-likelihood ratio is
$$
\ell(x)=\log\frac{p(x)}{q(x)}=-\delta x+\frac{\delta^{2}}{2},
$$
and the KL divergence is $\mathrm{KL}(Q\|P)=\delta^{2}/2$.  
Since $\ell(X)$ has unbounded support under both $P$ and $Q$, any finite clamp $[a,b]$ is nontrivial, and thus Proposition~\ref{prop: AREofncLLR} yields $\mathrm{ARE}(\mathrm{ncLLR}:\mathrm{LLR})<1$.
\end{example}

\gdpSimpleHT*

\begin{proof}
    To set up the notation for mean, denote $\mathrm{LLR} (\underline{x};P,Q) = n \bar{\ell}_n$, where $\underline{\ell_n} = (\ell_1,\ldots,\ell_n)$ following the same convention as $\underline{x}$. Since Assumption (A.1) and (A.2) is posed on the statistic $\ell(x_i; f_{\theta_0},f_{\theta_1})$ for all $i \in [n]$, 
    \begin{align*}
        |\mu_{DP}(\underline{\ell_n}, \ep) - \bar{\ell}_n| = \widetilde{O}_p \left( \frac{1}{n} \right).
    \end{align*}
    By the {Neyman--Pearson Lemma}, we know that
    $$
    \psi(\underline{x}) =
    \begin{cases}
    1, & \bar{\ell}_n > \lambda_n\\
    \gamma, & \bar{\ell}_n = \lambda_n\\
    0, & \bar{\ell}_n < \lambda_n
    \end{cases}
    $$
    is the most powerful level-$\alpha$ test for $H_0: \theta=\theta_0$ versus $H_1: \theta=\theta_1$ so that $P_{\theta_0}(\bar{\ell}_n > \lambda_n) = \alpha$. 
    For $\mu_{DP}(\underline{\ell_n}, \ep)$, we can similarly construct a private level-$\alpha$ test
    $$
    \phi(\underline{x}) =
    \begin{cases}
    1, & \mu_{DP}(\underline{\ell_n},\ep) > k(n,\ep)\\
    \gamma, & \mu_{DP}(\underline{\ell_n},\ep) = k(n,\ep)\\
    0, & \mu_{DP}(\underline{\ell_n},\ep) > k(n,\ep)
    \end{cases}
    $$
    such that $P_{\theta_0}(\mu_{DP}(\underline{\ell_n}, \ep) > k(n,\ep)) = \alpha$. By Proposition~\ref{prop: both algorithms are GDP}, we know that $\phi$ satisfies $\ep$-GDP. 

    By Theorem~\ref{thm: GDP-MeanEst Utility}, we have known that $\mu_{DP}(\underline{\ell_n}, \ep)$ and $\bar{\ell}_n$ have the same limiting distribution. Under the null hypothesis $H_0$, we can write that $\sqrt{n}(\bar{t}_n - \theta_0) \to N(0,\sigma_0^2)$ and hence $\sqrt{n}(\mu_{DP}(\underline{\ell_n}, \ep) - \theta_0) \to N(0,\sigma_0^2)$. Now, let $T_{1,n} = \mu_{DP}(\underline{\ell_n}, \ep),\ T_{2,n} = \bar{\ell}_n$ be two sequences of statistics. Then, for $i=1,2$, let $e_i(y) ={(y-\theta_0)^2}/{\sigma_0^2}$ and $\mu_i(\theta) = \theta$, we have
    \begin{align*}
        \frac{-2}{n} \log P_{\theta_0} (T_{i,n} \ge y) & \to e_i(y) \\
        T_{i,n} &\overset{P_\theta}{\to} \mu_i(\theta),
    \end{align*}
    where the first convergence can be seen from the sub-Gaussian concentration inequality and the second convergence follows from the \emph{law of large numbers}.
     With the conditions of Theorem~14.22 in \citet{van2000asymptotic} satisfied, the Bahadur efficiency is 
    \begin{align*}
        \frac{e_1(\mu_1(\theta_1))}{e_2(\mu_2(\theta_1))}
        = \frac{(\theta_1-\theta_0)^2/\sigma_0^2}{(\theta_1-\theta_0)^2/\sigma_0^2} 
        = 1.
    \end{align*}
\end{proof}

Leveraging Theorem~\ref{thm: GDP-MeanEst Utility} and Theorem~14.22 in \citet{van2000asymptotic}, Theorem~\ref{thm: gdpSimpleHT} shows that $\mu_{DP}(\underline{\ell};\ep)$ satisfies Gaussian differential privacy and attains asymptotic relative efficiency equal to one in the simple-hypothesis setting.

\gdpOneSidedHT*

\begin{proof}
    To set up the notation of mean carefully, the non-private sufficient statistic for data size $n$ is $\bar{t}_n$ and the private counterpart is $\mu_{DP}(\underline{t_n}, \ep)$. Since Assumption (A.1) and (A.2) is posed on $t(x_i)$ for all $i \in [n]$, 
    \begin{align*}
        |\mu_{DP}(\underline{t_n}, \ep) - \bar{t}_n| = \widetilde{O}_p \left( \frac{1}{n} \right).
    \end{align*}
    By the {Neyman--Pearson Lemma} and the {Karlin--Rubin Theorem}, we know that
    $$
    \psi(\underline{x}) =
    \begin{cases}
    1, & \bar{t}_n  > \lambda_n\\
    \gamma, & \bar{t}_n = \lambda_n\\
    0, & \bar{t}_n < \lambda_n
    \end{cases}
    $$
    is the uniformly most powerful level-$\alpha$ test for $H_0: \theta \le \theta_0$ versus $H_1: \theta > \theta_0$ such that $P_{\theta_0}(\bar{t}_n > \lambda_n) = \alpha$, where $\lambda_n = \Lambda/n$. 
    For $\mu_{DP}(\underline{t_n}, \ep)$, we can similarly construct a private level-$\alpha$ test
    $$
    \phi(\underline{x}) =
    \begin{cases}
    1, & \mu_{DP}(\underline{t_n}, \ep) > k(n,\ep)\\
    \gamma, & \mu_{DP}(\underline{t_n}, \ep) = k(n,\ep)\\
    0, & \mu_{DP}(\underline{t_n}, \ep) < k(n,\ep)
    \end{cases}
    $$
    such that $P_{\theta_0}(\mu_{DP}(\underline{t_n}, \ep) > k(n,\ep)) = \alpha$. By Proposition~\ref{prop: both algorithms are GDP}, we know that $\phi$ satisfies $\ep$-GDP. 


    By Theorem~\ref{thm: GDP-MeanEst Utility}, we have known that $\mu_{DP}(\underline{t_n}, \ep)$ and $\bar{t}_n$ have the same limiting distribution. Under the null hypothesis $H_1$, we can write that $\sqrt{n}(\bar{t}_n - \theta_0) \to N(0,\sigma_0^2)$ and hence $\sqrt{n}(\mu_{DP}(\underline{t_n}, \ep) - \theta_0) \to N(0,\sigma_0^2)$. Now, let $\theta_n = \theta_0 + h/\sqrt{n}$ be a real sequence that converges to $\theta_0$ and $T_{1,n} = \mu_{DP}(\underline{t_n}, \ep),\ T_{2,n} = \bar{t}_n$ be two sequences of statistics. Then, for $i=1,2$,
    \begin{align*}
        \frac{\sqrt{n}(T_{i,n} - \theta_n)}{\sigma_{i,n}} \to N(0,1),
    \end{align*}
    and $\sigma_{i,n} \to \sigma_0$. With the conditions of Theorem~14.19 in \citet{van2000asymptotic} satisfied, the Pitman efficiency is 
    \begin{align*}
        \lim_{n\to\infty} \left( \frac{1/\sigma_{1,n}}{1/\sigma_{2,n}} \right) = 1.
    \end{align*}
    
\end{proof}

Leveraging Theorem~\ref{thm: GDP-MeanEst Utility} and Theorem~14.19 of \citet{van2000asymptotic}, Theorem~\ref{thm: gdpOneSidedHT} shows that $\mu_{DP}(\underline{t};\ep)$ satisfies Gaussian differential privacy and attains asymptotic relative efficiency equal to one in the one-sided-hypothesis setting. 

For the result of the two-sided hypothesis (Theorem~\ref{thm: gdpTwoSidedHT}), we first prove the following two lemmas.

\begin{lemma} \label{lem: two-sided Pitman efficiency}
Consider statistical models $\{P_{n,\theta} : \theta \ge \theta_0\}$ such that 
$\|P_{n,\theta} - P_{n,\theta_0}\| \to 0$ as $\theta \to \theta_0$ for every fixed $n$. 
Let $T_{1,n}$ and $T_{2,n}$ be sequences of statistics satisfying
$$
\frac{\sqrt{n}\,\bigl(T_{i,n} - \mu_i(\theta_n)\bigr)}{\sigma_i(\theta_n)}
\;\overset{\theta_n}{\to}\; N(0,1),
$$
for every sequence $|\theta_n - \theta_0|\to 0$, where the functions $\mu_i$ and $\sigma_i$ satisfy: $\mu_i$ is differentiable at $\theta_0$ with $\mu_i'(\theta_0) > 0$ and $\sigma_i$ is continuous at $\theta_0$ with $\sigma_i(\theta_0) > 0$. Suppose the corresponding two-sided tests reject $H_0 : \theta = \theta_0$ for large values of $|T_{i,n} - m_{i,n}|$, where the centering sequences satisfy $m_{i,n} - \mu_i(\theta_0) = o\!\left(1/\sqrt{n}\right)$. Then, for every sequence of local alternatives $|\theta_\nu - \theta_0| \to 0$, the Pitman asymptotic
relative efficiency of the two test sequences is
$$
\mathrm{ARE}(T_{1,n},T_{2,n})
=
\left(
\frac{\mu_1'(\theta_0)/\sigma_1(\theta_0)}
     {\mu_2'(\theta_0)/\sigma_2(\theta_0)}
\right)^2,
$$
independently of $\alpha > 0$ and $\beta \in (\alpha,1)$.
\end{lemma}

\begin{proof}
In the first part of \citet{van2000asymptotic}'s proof, it is shown by contradiction that $n_{\nu,i} \to \infty$ as $\nu \to \infty$. The contradiction arises because when $\theta_\nu \to \theta_0$, any fixed $n$ cannot distinguish $P_{n,\theta_\nu}$ from $P_{n,\theta_0}$, which can be seen from the fact that for fixed $n$ the sum of the type~I and type~II probabilities tends to $1$, contradicting that it must be bounded above by $\alpha + (1-\beta) < 1$. The same argument applies in the present two-sided
setting.

Once we have established that $n_{\nu,i} \to \infty$ as $\nu \to \infty$, the asymptotic normality of the test statistic $T_{i,n}$ applies. Since the limiting distribution is continuous, the minimal sample sizes required to achieve level at most $\alpha$ and power at least $\beta$ attain asymptotic level exactly $\alpha$ and asymptotic power exactly $\beta$. For the two-sided tests that reject $H_0$ for large values of $|T_{i,n} - \mu_i(\theta_0)|$, the unbiasedness condition, together with the asymptotic symmetry of the limiting normal distribution around $\mu_i(\theta_0)$, implies that asymptotically the upper and lower tails each carry probability $\alpha/2$ under $H_0$. Therefore the asymptotic level-$\alpha$ test rejects the null if
\begin{align*}
  T_{n_\nu,i}
  &> \mu_i(\theta_0)
      + \frac{\sigma_i(\theta_0) z_{\alpha/2}}{\sqrt{n_\nu}}
      + r_{n_\nu,i}^a
    \;=:\; c_{n_\nu,i}^a,\\
  T_{n_\nu,i}
  &< \mu_i(\theta_0)
      - \frac{\sigma_i(\theta_0) z_{\alpha/2}}{\sqrt{n_\nu}}
      + r_{n_\nu,i}^b
    \;=:\; c_{n_\nu,i}^b,
\end{align*}
where $r_{n_\nu,i}^a, r_{n_\nu,i}^b = o(1/\sqrt{n_\nu})$.
Let $m_{n_\nu,i} = (c_{n_\nu,i}^a + c_{n_\nu,i}^b)/2$, then the test rejects for large values of $|T_{n_\nu,i}-m_{n_\nu,i}|$. While $m_{n_\nu,i}$ may not equal $\mu_i(\theta_0)$, their difference $m_{n_\nu,i} - \mu_i(\theta_0) = (r_{n_\nu,i}^a + r_{n_\nu,i}^b)/2 = o(1/\sqrt{n_\nu})$. Thus, by Slutsky's theorem,
$$
\sqrt{n_\nu}\bigl(T_{n_\nu,i} - m_{n_\nu,i}\bigr) = \sqrt{n_\nu}\bigl(T_{n_\nu,i} - \mu_i(\theta_0)\bigr) + o(1)
$$
so centering at $m_{n_\nu,i}$ or at $\mu_i(\theta_0)$ is asymptotically equivalent. Under the local alternatives $\theta_\nu = \theta_0 + h/\sqrt{n_\nu}$, the LAN expansion yields
$$
\sqrt{n_\nu}\bigl(T_{n_\nu,i} - \mu_i(\theta_0)\bigr)
\to N\!\bigl(h\,\mu_i'(\theta_0),\,\sigma_i^2(\theta_0)\bigr),
$$
so the standardized statistic converges to $N(\lambda_i(h),1)$, where $\lambda_i(h) = h \mu_i'(\theta_0) / \sigma_i(\theta_0)$.
Therefore, the local powers of these tests are
$$
\beta_{n_{\nu},i}(\theta_{\nu})
=
1 - \Phi\!\left(z_{\alpha/2}
                - \sqrt{n_{\nu}} \theta_{\nu}
                  \frac{\mu_{i}^{\prime}(\theta_0)}{\sigma_{i}(\theta_0)}\right)
  + \Phi\!\left(-z_{\alpha/2}
                - \sqrt{n_{\nu}} \theta_{\nu}
                  \frac{\mu_{i}^{\prime}(\theta_0)}{\sigma_{i}(\theta_0)}\right)
  + o(1).
$$
Define
$$
g(\lambda)
=
1 - \Phi\bigl(z_{\alpha/2}-\lambda\bigr)
  + \Phi\bigl(-z_{\alpha/2}-\lambda\bigr).
$$
Then $g$ is continuous and strictly increasing, so
$$
\beta_{n_\nu,i}(\theta_\nu) \to \beta
\quad\Longleftrightarrow\quad
g \Bigl(\sqrt{n_{\nu}} \theta_{\nu}
          \frac{\mu_{i}^{\prime}(\theta_0)}{\sigma_{i}(\theta_0)}\Bigr)
\to \beta
\quad\Longleftrightarrow\quad
\sqrt{n_{\nu}} \theta_{\nu}
\frac{\mu_{i}^{\prime}(\theta_0)}{\sigma_{i}(\theta_0)}
\to \lambda_\beta,
$$
where $\lambda_\beta$ is the unique solution of $g(\lambda_\beta)=\beta$. Squaring both sides gives
$$
n_{\nu,i} \theta_{\nu}^2
\;\longrightarrow\;
\frac{\lambda_\beta^2}
     {\bigl(\mu_{i}^{\prime}(\theta_0)/\sigma_{i}(\theta_0)\bigr)^2}.
$$
Applying this for $i=1,2$ and taking the ratio, we obtain
\begin{align*}
  \lim_{\nu\to \infty} \frac{n_{\nu,2}}{n_{\nu,1}} 
  &= \frac{\lambda_\beta^2/(\mu_2^\prime(\theta_0)/\sigma_2(\theta_0))^2}
          {\lambda_\beta^2/(\mu_1^\prime(\theta_0)/\sigma_1(\theta_0))^2} \\
  &= \left(\frac{\mu_1^{\prime}(\theta_0)/\sigma_1(\theta_0)}
                {\mu_2^{\prime}(\theta_0)/\sigma_2(\theta_0)}\right)^2,
\end{align*}
which is the asserted relative efficiency.
\end{proof}

\begin{lemma}[Two-sided cutoff proximity] \label{lem: two-sided cutoff proximity}
Let $\bar t_n$ be the non-private sufficient statistic and let $\mu_{DP}(\underline{t_n}, \ep)$ be its $\varepsilon$-GDP private counterpart from \texttt{GDP-MeanEst} such that $\mu_{DP}(\underline{t_n}, \ep)=\bar t_n+\Delta_n$ with $\Delta_n=\widetilde O_p(1/n)$. Let $(\lambda_{l}(n),\lambda_{u}(n))$ be the lower and upper cutoffs of the non-private UMPU level-$\alpha$ test based on $\bar t_n$ such that
\begin{align*}
    \ex_{\theta_0} \psi = \alpha \text{ and } \ex_{\theta_0} [\psi t] = \alpha \ex_{\theta_0} t,
\end{align*}
and let $(k_{l}(n,\ep),k_{u}(n,\ep))$ be the corresponding cutoffs of the private level-$\alpha$ test based on
$\mu_{DP}(\underline{t_n}, \ep)$ such that
\begin{align*}
    \ex_{\theta_0} [1\{\mu_{DP}(\underline{t_n},\ep) < k_l(n,\ep)\}] = \ex_{\theta_0} [1\{\mu_{DP}(\underline{t_n},\ep) > k_u(n,\ep)\}] = \frac{\alpha}{2} + o(1).
\end{align*}
Then,
$$
k_{l}(n,\ep)-\lambda_{l}(n)=\widetilde O(1/n),
\qquad
k_{u}(n,\ep)-\lambda_{u}(n)=\widetilde O(1/n).
$$
\end{lemma}

\begin{proof}
Let $F_n(y)=P_{\theta_0}(\bar t_n\le y)$ denote the cumulative distribution function of $\bar t_n$ under $H_0$.
By the level-$\alpha$ constraints for the non-private and private tests, we have
\begin{align} \label{eq: non-private level alpha}
\alpha = F_n(\lambda_{l}(n))+1-F_n(\lambda_{u}(n)).    
\end{align}
and
\begin{align}
\alpha
&=P_{\theta_0}\!\left(\mu_{DP}(\underline{t_n}, \ep)<k_{l}(n,\ep)\right) + P_{\theta_0}\!\left(\mu_{DP}(\underline{t_n}, \ep)>k_{u}(n,\ep)\right) \nonumber \\
&= P_{\theta_0}\!\left(t_n < k_{l}(n,\ep) - \Delta_n \right) + P_{\theta_0}\!\left(t_n > k_{u}(n,\ep) - \Delta_n\right) \nonumber \\
&= \mathbb E_{\Delta_n}\!\left[ F_n(k_{l}(n,\ep)-\Delta_n \,|\, \Delta_n) + 1 - F_n(k_{u}(n,\ep)-\Delta_n \,|\, \Delta_n)\right] \label{eq: law of total expectation on Fn} \\
&= \mathbb E_{\Delta_n}\!\left[ F_n(k_{l}(n,\ep))-f_n(\xi_{l}(n))\Delta_n) + 1 - \left( F_n(k_{u}(n,\ep))-f_n(\xi_{u}(n))\Delta_n \right) \right] \label{eq: mean-value on Fn} \\
&= F_n(k_{l}(n,\ep))-f_n(\xi_{l}(n)) \ex_{\Delta_n}[\Delta_n] + 1 - \left( F_n(k_{u}(n,\ep))-f_n(\xi_{u}(n)) \ex_{\Delta_n}[\Delta_n] \right) \label{eq: private level alpha},
\end{align}
where $\xi_{l}(n)$ is between $k_{l}(n,\ep)-\Delta_n$ and $k_{l}(n,\ep)$ and $\xi_{u}(n)$ is between $k_{u}(n,\ep)-\Delta_n$ and $k_{u}(n,\ep)$. Note that \eqref{eq: law of total expectation on Fn} holds by the law of total expectation, \eqref{eq: mean-value on Fn} holds from applying the mean-value theorem pointwise, $F_n(k_{l}(n,\ep)-\Delta_n) = F_n(k_{l}(n,\ep))-f_n(\xi_{l}(n))\Delta_n$, and $F_n(k_{u}(n,\ep)-\Delta_n) = F_n(k_{u}(n,\ep))-f_n(\xi_{u}(n))\Delta_n$.

Subtracting \eqref{eq: private level alpha} by \eqref{eq: non-private level alpha} yields $\bigl[F_n(k_{l}(n,\ep))-F_n(\lambda_{l}(n))\bigr] - \bigl[F_n(k_{u}(n,\ep))-F_n(\lambda_{u}(n))\bigr] = \bigl[f_n(\xi_{l}(n))-f_n(\xi_{u}(n))\bigr]\,\mathbb E(\Delta_n)$, which gives
\begin{align} \label{eq: expansion-delta}
    \bigl[F_n(k_{l}(n,\ep))-F_n(\lambda_{l}(n))\bigr] - \bigl[F_n(k_{u}(n,\ep))-F_n(\lambda_{u}(n))\bigr] = \widetilde O(1/n)
\end{align}
because $\Delta_n=\widetilde O_p(1/n)$ and the density $f_n$ is assumed to be bounded.

Applying the mean-value theorem again to $F_n$, we have $F_n(k_{l}(n,\ep))-F_n(\lambda_{l}(n)) = f_n(\zeta_{l}(n))(k_{l}(n,\ep)-\lambda_{l}(n))$ and $F_n(k_{u}(n,\ep))-F_n(\lambda_{u}(n)) = f_n(\zeta_{u}(n))(k_{u}(n,\ep)-\lambda_{u}(n))$ for some $\zeta_{l}(n)$ between $k_{l}(n,\ep)$ and $\lambda_{l}(n)$ and $\zeta_{u}(n)$ between $k_{u}(n,\ep)$ and $\lambda_{u}(n)$. Thus, \eqref{eq: expansion-delta} becomes
\begin{align}
f_n(\zeta_{l}(n))(k_{l}(n,\ep)-\lambda_{l}(n))
-
f_n(\zeta_{u}(n))(k_{u}(n,\ep)-\lambda_{u}(n))
=
\widetilde O(1/n).
\label{eq: linear-system-1}
\end{align}

Now, the exact unbiasedness of the UMPU test $\psi$ and the asymptotic unbiasedness of the private test $\phi$ imply that the corresponding unbiasedness moment constraint holds up to a vanishing remainder, offering an additional condition on $k_{l}(n,\ep)-\lambda_{l}(n)$ and $k_{u}(n,\ep)-\lambda_{u}(n)$:
$ \mathbb E_{\theta_0}\bigl[(\psi-\phi)\,\bar t_n\bigr] = \int_{[k_{l}(n,\ep),\lambda_{l}(n)]\cup[k_{u}(n,\ep),\lambda_{u}(n)]} t f_{\theta_0}(t) dt = \widetilde O(1/n)$.
Applying the mean-value theorem gives
\begin{align} \label{eq: linear-system-2}
g(\varrho_{a,n})(k_{l}(n,\ep)-\lambda_{l}(n)) + g(\varrho_{b,n})(k_{u}(n,\ep)-\lambda_{u}(n)) = \widetilde O(1/n),
\end{align}
where $\varrho_{a,n}\in[k_{l}(n,\ep),\lambda_{l}(n)]$,
$\varrho_{b,n}\in[k_{u}(n,\ep),\lambda_{u}(n)]$, and $g(t)=t f_{\theta_0}(t)$. Equations \eqref{eq: linear-system-1} and \eqref{eq: linear-system-2} form a linear
system in $(k_{l}(n,\ep)-\lambda_{l}(n),\,k_{u}(n,\ep)-\lambda_{u}(n))$ with coefficients bounded
away from zero. Solving the system yields $k_{l}(n,\ep)-\lambda_{l}(n)=\widetilde O(1/n)$ and $k_{u}(n,\ep)-\lambda_{u}(n)=\widetilde O(1/n)$ as claimed.
\end{proof}

\gdpTwoSidedHT*

\begin{proof}
We follow the notation in the proof of Theorem~\ref{thm: gdpOneSidedHT}. Let $\bar t_n$ denote the non-private sufficient statistic based on $n$ observations, and let $\mu_{DP}(\underline{t_n}, \ep)$ be its $\varepsilon$-GDP private counterpart.
By Assumption~(A.1) and (A.2) on the sufficient statistic $t(x_i)$ and Theorem~\ref{thm: GDP-MeanEst Utility}, we can write $\mu_{DP}(\underline{t_n}, \ep)=\bar t_n+\Delta_n$ with $\Delta_n=\widetilde O_p(1/n)$. Let
$$
\psi(\underline{x})=
\begin{cases}
1, & \bar t_n<\lambda_{l}(n)\ \text{or}\ \bar t_n>\lambda_{u}(n),\\
\rho_a, & \bar t_n=\lambda_{l}(n),\\
\rho_b, & \bar t_n=\lambda_{u}(n),\\
0, & \lambda_{l}(n)\le \bar t_n\le \lambda_{u}(n),
\end{cases}
$$
be the uniformly most powerful unbiased (UMPU) level-$\alpha$ test for $H_0:\theta=\theta_0$ versus $H_1:\theta\neq\theta_0$, so that $\ex_{\theta_0}\psi=\alpha$ and $\ex_{\theta}\psi\ge \alpha\ \ \text{for all }\theta\neq\theta_0$.
Define the private two-sided test
$$
\phi(\underline{x})=
\begin{cases}
1, & \mu_{DP}(\underline{t_n}, \ep)<k_{l}(n,\ep)\ \text{or}\ \mu_{DP}(\underline{t_n}, \ep)>k_{u}(n,\ep),\\
r_a, & \mu_{DP}(\underline{t_n}, \ep)=k_{l}(n,\ep),\\
r_b, & \mu_{DP}(\underline{t_n}, \ep)=k_{u}(n,\ep),\\
0, & k_{l}(n,\ep)\le \mu_{DP}(\underline{t_n}, \ep)\le k_{u}(n,\ep),
\end{cases}
$$
with $(k_{l}(n,\ep),k_{u}(n,\ep),r_a,r_b)$ chosen such that $\lim_{n \to \infty}\ex_{\theta_0}\phi=\alpha$.
By Proposition~\ref{prop: both algorithms are GDP}, the test $\phi$ satisfies $\varepsilon$-GDP.

Unlike $\psi$, the test $\phi$ need not be exactly unbiased in finite samples because its rejection region is defined through the perturbed statistic $\mu_{DP}(\underline{t_n}, \ep)=\bar t_n+\Delta_n$. Nevertheless, since $\Delta_n=\widetilde O_p(1/n)$, the perturbation vanishes on the $\sqrt n$-scale relevant for Pitman local alternatives, and $\phi$ is asymptotically unbiased in the sense that any bias in its local power expansion (equivalently, any deviation from the derivative-zero condition at $\theta_0$) is of smaller order and does not affect the Pitman efficiency calculation.

Now, we define the centering sequences
$$
m_{1,n}:=\frac{k_{l}(n,\ep)+k_{u}(n,\ep)}{2},
\qquad
m_{2,n}:=\frac{\lambda_{l}(n)+\lambda_{u}(n)}{2}.
$$
First, for the UMPU test $\psi$, the exact unbiasedness together with the asymptotic normality (and hence asymptotic symmetry) of $\bar t_n$ under $H_0$ implies that the two-sided rejection region is asymptotically centered at $\theta_0$, and more precisely that
$$
m_{2,n}-\theta_0=O(1/n).
$$
Equivalently, under $H_0$ one has $\lambda_{l}(n)=\theta_0-\sigma_0 z_{\alpha/2}/\sqrt n+O(1/n)$ and
$\lambda_{u}(n)=\theta_0+\sigma_0 z_{\alpha/2}/\sqrt n+O(1/n)$, so averaging yields the stated bound. Next, by the cutoff comparison result in Lemma~\ref{lem: two-sided cutoff proximity},
$$
m_{1,n}-m_{2,n}
=
\frac{(k_{l}(n,\ep)-\lambda_{l}(n))+(k_{u}(n,\ep)-\lambda_{u}(n))}{2}
=
\widetilde O(1/n).
$$
Combining the preceding two displays gives, for $i=1,2$,
\begin{align} \label{eq: centering condition}
    m_{i,n}-\theta_0= \widetilde O\!\left(\frac{1}{n}\right),
\end{align}
which is exactly the centering condition required to apply Lemma~\ref{lem: two-sided Pitman efficiency}.

We are now ready to conclude the Pitman efficiency. By Theorem~\ref{thm: GDP-MeanEst Utility}, $\mu_{DP}(\underline{t_n}, \ep)$ and $\bar t_n$ have the same limiting distribution. Under $H_0$, $\sqrt n(\bar t_n-\theta_0)\overset{d}{\to} N(0,\sigma_0^2)$ and hence $\sqrt n\bigl(\mu_{DP}(\underline{t_n}, \ep)-\theta_0\bigr) \overset{d}{\to}N(0,\sigma_0^2)$. Let $|\theta_n-\theta_0|=h/\sqrt n$ be a sequence of local alternatives and set $T_{1,n}=\mu_{DP}(\underline{t_n}, \ep)$ and $T_{2,n}=\bar t_n$. Then, for $i=1,2$,
$$
\frac{\sqrt n\,(T_{i,n}-\theta_n)}{\sigma_{i,n}}\overset{d}{\to}N(0,1),
$$
and $\sigma_{i,n}\to\sigma_0$ and by \eqref{eq: centering condition} the corresponding two-sided tests reject for large values of $|T_{i,n}-m_{i,n}|$ with $m_{i,n}-\theta_0=o(1/\sqrt n)$. Therefore the conditions of Lemma~\ref{lem: two-sided Pitman efficiency} are satisfied, and the Pitman asymptotic relative efficiency is
$$
\lim_{n\to\infty}\left(\frac{1/\sigma_{1,n}}{1/\sigma_{2,n}}\right)^2
=
1.
$$
This completes the proof.
\end{proof}

Leveraging Lemma~\ref{lem: two-sided Pitman efficiency}, \ref{lem: two-sided cutoff proximity} and Theorem~\ref{thm: GDP-MeanEst Utility}, Theorem~\ref{thm: gdpOneSidedHT} shows that $\mu_{DP}(\underline{t};\ep)$ satisfies Gaussian differential privacy and attains asymptotic relative efficiency equal to one in the two-sided-hypothesis setting.

\bibliography{bibliography} 

\end{document}